%% file: main_arxiv.tex
\documentclass{article}
 \pdfoutput=1
\usepackage{etoolbox}
\newtoggle{ieee}
\togglefalse{ieee}

\usepackage[margin=1in]{geometry}

\usepackage{zref-xr, zref-user} % external references
\usepackage{fancyhdr}

\input{usepackage}
\input{notation}
\fancypagestyle{plain}{% emulate the plain style
  \fancyhf{}% clear all fields
  \fancyfoot[C]{\thepage}%
}

\fancypagestyle{firstpage}{% a rule at the bottom
  \fancyhf{}% clear all fields
  \fancyfoot[L]{\footnotesize{\copyright~2020 IEEE. Personal use of this material is permitted.  Permission from IEEE must be obtained for all other uses, in any current or future media, including reprinting/republishing this material for advertising or promotional purposes, creating new collective works, for resale or redistribution to servers or lists, or reuse of any copyrighted component of this work in other works.}}
}

\title{A Successive-Elimination Approach to Adaptive Robotic Sensing}

\author{
  Esther Rolf\thanks{UC Berkeley, CA. erolf@berkeley.edu}
\and
  David Fridovich-Keil\thanks{UC Berkeley, CA. dfk@berkeley.edu}
\and
  Max Simchowitz\thanks{UC Berkeley, CA. msimchow@berkeley.edu} 
\and
  Benjamin Recht\thanks{UC Berkeley, CA. brecht@berkeley.edu} 
\and
  Claire Tomlin\thanks{UC Berkeley, CA. tomlin@eecs.berkeley.edu}}

\begin{document}

\maketitle
\thispagestyle{firstpage}

\begin{abstract}
  \input{abstract}
\end{abstract}

% ==================== INTRODUCTION ==================
\input{introduction}

% ==================== RELATED WORK ==================
\input{related_work}

% =================== ALGORITHM ==================
%\input{algorithm_heirarchical}
\input{algorithm_snake}
\input{radioactive_source_seeking}
\input{theorem_and_baby_proof}
% =================== METHODS ==================
%\input{methods}

% ====================== RESULTS =====================
\input{results}
\input{generalizations}
% ===================== CONCLUSIONS ===================
\input{conclusions}

% ===================== CONCLUSIONS ===================

\section*{Acknowledgments}
%We'd like to thank ... \\
The authors would like to thank the anonymous reviewers of the journal publication of this paper for comments and suggestions. Additionally, we thank Andrew Haefner for thoughtful insights on the experimental setup and sensing model.
This material is based upon work supported by the National Science Foundation Graduate Research Fellowship
under Grant No. DGE 1752814.

%% Use plainnat to work nicely with natbib.
%\pagebreak
%\bibliographystyle{plain}
%\bibliographystyle{IEEEtran}
\bibliographystyle{unsrt}
\bibliography{references}

\pagebreak
\appendix

{\noindent \Huge \bf Appendix }

\vspace{1em}
{Note: this appendix was not peer reviewed as part of the journal publication of this paper. We include it in this report for completeness.

\input{lemma_one_proof}

\input{analysis_intro}
\input{top_k_results}
\input{uniform_analysis}

%\input{conf_interval}
\input{analysis_snake}
\input{analysis_naive}
%\input{top_k_proofs}
\input{concentration_proofs}
\input{lower_bounds}

%\pagebreak

%\pagebreak

%\section*{Extra things.}

%\section*{Snake Algorithm}
%\input{algorithm_snake}
%\end{document}

\end{document}

%% file: usepackage.tex
\usepackage[utf8]{inputenc}
\usepackage{graphicx}
\usepackage{array}
\usepackage{epsfig}
\usepackage[T1]{fontenc}

\usepackage{amsfonts,amssymb,amsmath,bm}
\usepackage{gensymb}

\usepackage{subcaption}
\usepackage{cite}
\usepackage{color}
\usepackage[ruled,vlined,linesnumbered]{algorithm2e}
\usepackage{algpseudocode}

\usepackage{url}

\graphicspath{{figures/}}
\DeclareGraphicsExtensions{.pdf,.jpeg,.png,.eps}

\makeatletter

\newcommand{\squeeze}[1]{{\vspace{-0.1cm}}}
\newcommand{\remove}[1]{}

\usepackage[T1]{fontenc}
\usepackage[scaled=0.85]{beramono}

\usepackage{hyperref}       % hyperlinks
\usepackage{url}            % simple URL typesetting
\usepackage{booktabs}       % professional-quality tables
\usepackage{amsfonts}       % blackboard math symbols
\usepackage{nicefrac}       % compact symbols for 1/2, etc.
\usepackage{microtype}      % microtypography
\usepackage{amsthm}

\usepackage{thm-restate}

\usepackage{cleveref}

\iftoggle{ieee}{}
{\usepackage{amsthm}
}

%% file: notation.tex
\newcommand{\problemname}{\text{source seeking}}

\DeclareMathOperator*{\argmax}{arg\,max}

{}

\newcommand{\trajpoints}{{\config}}

\newcommand{\RSS}{\text{RSS}}

\newcommand{\Tsample}{T^{\mathrm{sample}}}
\newcommand{\Truntime}{T^{\mathrm{run}}}

\newcommand{\config}{\mathcal{Z}}
\newcommand{\R}{\mathbb{R}}

\newcommand{\deltot}{\delta_{\mathrm{tot}}} %time step

\newcommand{\signal}{\mathbf{X}}
\newcommand{\signaly}{\mathbf{Y}}

\newcommand{\cnt}{\mathbf{N}}
\newcommand{\RR}{\mathbb{R}}
\newcommand{\To}{\longrightarrow}

\def\Vec#1{\!\!\hbox{$#1$\kern-0.38em\lower0.85em\hbox{$\vec{}\,$}}\,}%
\newcommand{\bbm}{\begin{bmatrix}}
\newcommand{\ebm}{\end{bmatrix}}
\newcommand{\bpm}{\begin{pmatrix}}
\newcommand{\epm}{\end{pmatrix}}

\newcommand{\mbb}[1]{{\mathbb{#1}}}
\newcommand{\mc}[1]{\mathcal{#1}}

\newcommand{\normal}[1]{\mc{N}\left(#1\right)}

\newtheorem{definition}{Definition}

\newtheorem{theorem}{Theorem}
\newtheorem{lemma}[theorem]{Lemma}
\newtheorem{fact}[theorem]{Fact}

\newtheorem{corollary}[theorem]{Corollary}
\newtheorem{proposition}[theorem]{Proposition}

\newcommand{\calO}{\mathcal{O}}

\newcommand{\muhat}{\widehat{\mu}}

\newcommand{\Stop}{\calS^{\mathrm{top}}}

\newcommand{\Sstk}{\mathcal{S}^{*}(k)}

\newcommand{\calS}{\mathcal{S}}
\newcommand{\calSi}{\calS_i}
\newcommand{\Poi}{\mathrm{Poisson}}
\newcommand{\Egood}{\mathcal{E}_{\mathrm{good}}}
\newcommand{\mubar}{\overline{\mu}}

\newcommand{\UCB}{\mathrm{UCB}}
\newcommand{\LCB}{\mathrm{LCB}}

\newcommand{\UCBbar}{\overline{\UCB}}
\newcommand{\LCBbar}{\overline{\LCB}}

\newcommand{\calT}{\mathcal{T}}
\newcommand{\BigOhTil}{\widetilde{\mathcal{O}}}

\newcommand{\I}{\mathbb{I}}

\newcommand{\calE}{\mathcal{E}}

\newcommand{\snakeucbtext}{AdaSearch }
\newcommand{\snakeucb}{\mathtt{AdaSearch}}

\newcommand{\infomax}{\mathtt{InfoMax}}
\newcommand{\naivesnake}{\mathtt{NaiveSearch}}

\newcommand{\Hadapt}{\mathcal{C}_{\mathtt{adapt}}}
\newcommand{\Hunif}{\mathcal{C}_{\mathtt{unif}}}

\newcommand{\Hadaptk}{\Hadapt^{(k)}}
\newcommand{\Hunifk}{\Hunif^{(k)}}
\newcommand{\cardS}{|\calS|}

\newcommand{\Shat}{\widehat{\calS}}
\newcommand{\KL}{\mathrm{KL}}

\newcommand{\delbar}{\delta_0}
\newcommand{\impby}{\text{implied by}}

\newcommand{\diverg}{\mathrm{d}}
\newcommand{\divergx}{\diverg(\mu(x),\mu^*)}
\newcommand{\divergbot}{\diverg(\mu(x),\mu^{(k)})}
\newcommand{\divergtop}{\diverg(\mu^{(k+1)},\mu(x))}

\newcommand{\divergboteps}{\diverg_{\epsilon}(\mu(x),\mu^{(k)})}
\newcommand{\divergtopeps}{\diverg_{\epsilon}(\mu^{(k+1)},\mu(x))}

\newcommand{\mux}{\mu(x)}

\newcommand{\ifin}{i_{\mathrm{fin}}}

%% file: abstract.tex
%!TEX root = main_ieee.tex
We study an adaptive source seeking problem,
in which a mobile robot must identify the strongest emitter(s) of a signal in an environment with background
emissions.
Background signals may be highly heterogeneous and can mislead algorithms that are based on receding horizon control.
We propose
$\snakeucb$, a general algorithm for adaptive source seeking in the face of heterogeneous background noise. $\snakeucb$ combines global trajectory planning with
principled confidence intervals in order to concentrate measurements in
promising regions while guaranteeing sufficient coverage of the entire
area.
Theoretical analysis shows that $\snakeucb$ confers gains over a uniform sampling strategy when the distribution of background signals is highly
variable.
Simulation experiments demonstrate that
when applied to the problem of
 radioactive source seeking, $\snakeucb$ outperforms
both uniform sampling and
a receding time horizon information-maximization approach based on the current literature.
We also demonstrate $\snakeucb$ in hardware, providing further evidence of its potential for real-time implementation.

%% file: introduction.tex
%!TEX root = main_ieee.tex

\section{Introduction}
\label{sec:introduction}

\iftoggle{ieee}{\IEEEPARstart{R}{obotic}}{Robotic} \problemname~is a problem domain in which %one or many autonomous robots
a mobile robot
must traverse an environment to
locate the maximal emitters of a signal of interest, usually in the presence of
background noise. Adaptive \problemname~involves adaptive sensing and active
information gathering, and encompasses several
well-studied problems in robotics, including the rapid identification of
accidental contamination leaks and radioactive sources \cite{vetter2018gamma,
  masarich2018radiation},
and finding individuals in search and rescue missions \cite{hoffmann2010mobile}. %In such settings, it is often critical to devise a sensing strategy that returns a correct solution as quickly as possible.
We consider a specific motivating application of radioactive source-seeking (RSS), in which a
UAV (Fig.~\ref{fig:drone}) must identify the $k$-largest radioactive emitters in
a planar environment, where $k$ is a user-defined parameter.
RSS is a particularly interesting instance of \problemname~due
to
the challenges posed by the highly heterogeneous background noise \cite{pahlajani2014error}. %, and to

A well-adopted methodology for approaching \problemname\ problems is information
maximization (see Sec.~\ref{sec:related_work}), in which measurements are collected in the most promising locations
following a receding planning horizon. Information maximization is appealing because it
favors measuring regions that are likely to contain the highest
emitters and avoids wasting time elsewhere. However, when operating in real-time, computational constraints necessitate
approximations such as limits on planning horizon and trajectory
parameterization. These limitations scale with size of the search
region and complexity of the sensor model and may cause the algorithm to be excessively
greedy,
 spending extra travel time tracking down false leads.

 \begin{figure}[t]
%    \centering \includegraphics[width=0.95\columnwidth]{fancy_front_figure.png}
\centering 
\iftoggle{ieee}{
\includegraphics[width=0.95\columnwidth]{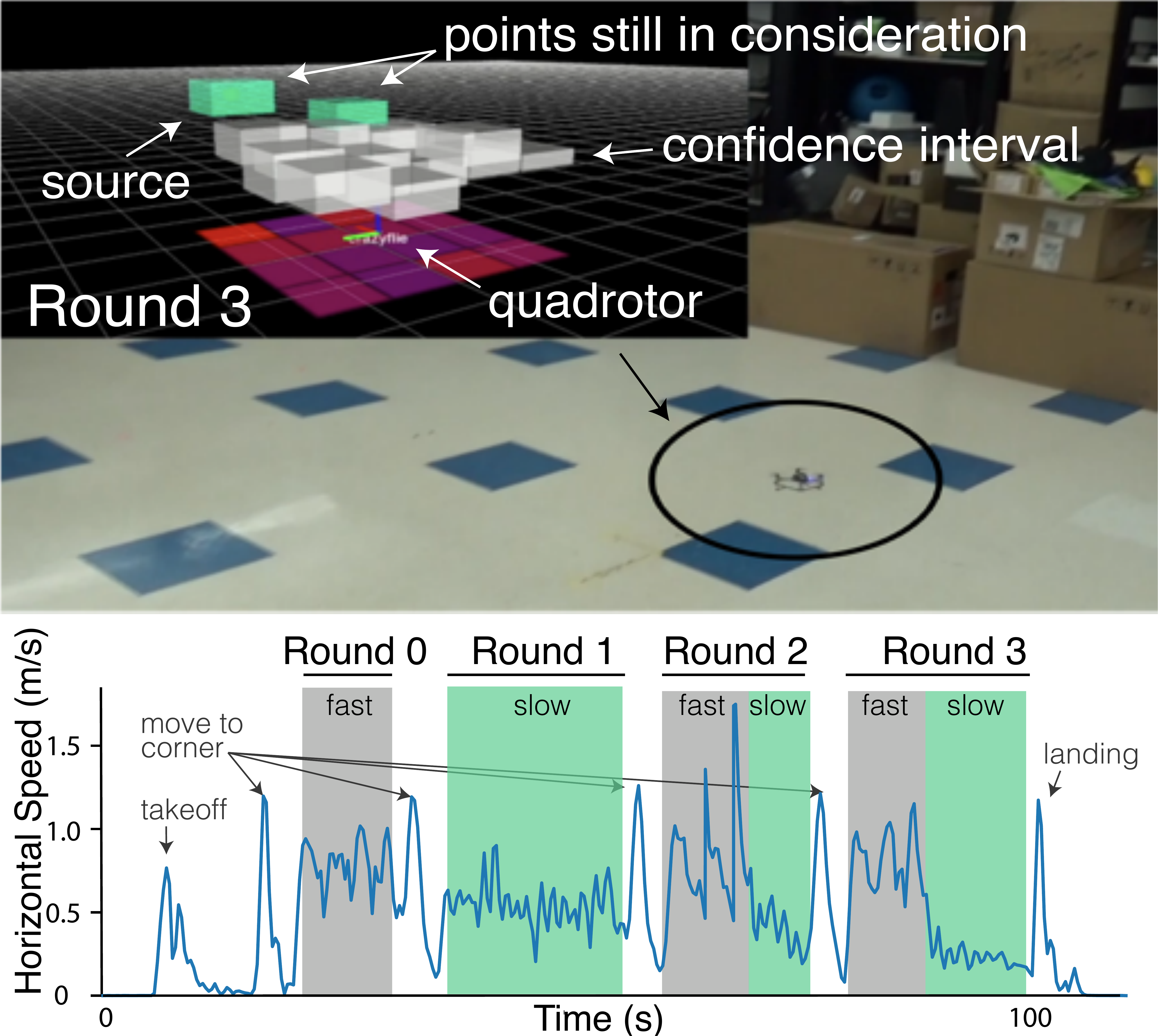}}
{\includegraphics[width=0.6\columnwidth]{figures_tro/front_fig_smaller.pdf}}
    \caption{A Crazyflie 2.0 quadrotor in a motion capture room during a hardware demo of $\snakeucb$ (round 3 of the algorithm). Top: confidence intervals for estimated radioactivity (simulated) within the room. 
    Bottom: horizontal speed over time, indicating slow and fast sections in
      which $\snakeucb$ allocates more and fewer sensor measurements, respectively. In later rounds, more time is spent measuring points that are still in consideration (green points).} %(red pixels are most
      %radioactive). Boxes reprepresent confidence intervals on the true emission rates;
     % green intervals are still in
      %consideration, while grey ones have been ruled out.}
    \label{fig:drone}
    \iftoggle{ieee}{\vspace{-1.5em}}{}
  \end{figure}

To overcome these limitations, we introduce $\snakeucb$, a successive-elimination framework for general
\problemname~problems with multiple sources, and demonstrate it within the context of RSS. $\snakeucb$ explicitly maintains confidence intervals over the emissions
rate at each point in the environment. Using these confidence intervals, the algorithm identifies a set of candidate points likely to be among the top-$k$ emitters, and eliminates points that are not.
Rather than iteratively planning for short, receding time horizons, $\snakeucb$ repeats a
\emph{fixed, globally-planned path}, adjusting the robot's speed in real-time to
focus measurements on promising regions.
%of the search space
% which contain the candidate points identified by our confidence intervals.
This approach offers
coverage of the full search space while affording an adaptive
measurement allocation in the spirit of information maximization. %In addition,
By maintaining a single fixed, global path, $\snakeucb$ reduces the online computational
overhead, yielding an algorithm easily amenable to real-time implementation. %Moreover, the explicit confidence intervals ensure a reliable termination criterion.
 %explicitly maintains confidence intervals over the emissions
%rate at each point in the environment. At each round, the agent adjusts its speed and, therefore, its measurement allocation, to .
%Moreover,
%In Sec.~\ref{sec:algorithm_snake}, we introduce the general $\snakeucb$ planning
%strategy, which in Sec.~\ref{sec:methods} we specialize to RSS with Poisson
%measurements and a linear sensing model . In
%Sec.~\ref{sec:results}, we compare $\snakeucb$ with an implementation of information
%maximization based on \cite{marchant2014bayesian} and tailored specifically to
%the radioactive source-seeking problem.

Specifically, our main contributions are:
\begin{itemize}
\item $\snakeucb$, a general framework for designing efficient sensing
  trajectories for robotic \problemname~problems,
\item Theoretical runtime analysis of $\snakeucb$ as well as of a naive, uniform
  sampling baseline which follows the same fixed global path but moves at
  constant speed, and
\item Simulation experiments for RSS evaluating $\snakeucb$ in comparison with a uniform baseline
  and information maximization.
\end{itemize}
Our theoretical analysis sharply quantifies $\snakeucb$'s improvement over its
uniform sampling analog. Experiments validate this finding in practice, and also show that $\snakeucb$ outperforms a custom implementation of information maximization tailored to the RSS problem.
Together, these results suggest that the accuracy and efficient runtime of $\snakeucb$ are robust to heterogeneous background noise, which stands in contrast to existing alternative methods. This robustness is particularly valuable in real-world applications where the exact distribution of background signals in the environment is likely unknown.
%Additionally, to emphasize
% that $\snakeucb$ is not only applicable in simulation, we present a hardware
%demonstration using a small quadrotor helicopter in a motion capture room.

The remainder of this paper is organized as follows. Sec.~\ref{sec:related_work}
presents a brief survey of related literature. Sec.~\ref{sec:algorithm_snake}
provides a formal statement of the \problemname~problem and introduces our
solution, $\snakeucb$. In Sec.~\ref{sec:methods}, we consider a radioactive
source seeking (RSS) case study and develop two appropriate sensing models which allow us to apply $\snakeucb$ to RSS.
%As points of theoretical and empirical comparison in the RSS case study, we consider a uniform sampling analog of $\snakeucb$ and a customized implementation of information maximization.
Sec.~\ref{sec:analysis} analyzes the theoretical runtime
complexity of $\snakeucb$ and its uniform sampling analog for the RSS problem. In Sec.~\ref{sec:results}, we present simulation experiments which corroborate these
theoretical results. %Our experiments show that $\snakeucb$ consistently
%performs favorably to both alternatives and, surprisingly, that the uniform sampling baseline performs comparably to information maximization.
A hardware
demonstration provides further evidence of $\snakeucb$'s potential for real-time
application. Sec.~\ref{sec:generalizations} suggests a number of extensions and generalizations to $\snakeucb$, and Sec.~\ref{sec:conclusion} concludes with a summary of our results.

%% file: related_work.tex
%!TEX root = main_ieee.tex

\section{Related Work}
\label{sec:related_work}

% Adaptive robotic sensing is central to problems in environmental monitoring~\cite{vetter2018gamma, masarich2018radiation,dunbabin2012robots}, and encompasses many different modeling choices and algorithmic methods.  %We characterize related work by underlying method, focusing on discussing aspects of methods which are  are ammenable for problem settings with large amounts of heterogenous noise.

There is a breadth of existing work related to \problemname.
Much of this literature, particularly when tailored to robotic
applications, leverages some form of information maximization, often using a
Gaussian process prior. However, our own work is inspired by approaches from the pure exploration
multi-armed bandit literature, even though bandits are not typically used to
model physical sensing problems with realistic motion constraints. We survey the
most relevant work in both information maximization and multi-armed bandits
below.

%\textbf{Information maximization methods.}
\subsection{Information maximization methods}
A popular approach to active sensing and source seeking in robotics, e.g. in active mapping \cite{bourgault2002information}
and target localization \cite{miller2016ergodic}, is to choose trajectories
that maximize a measure of information gain~\cite{bai2016information,ma2017active, charrow2015information,bourgault2002information,levine2010information}. In the specific case of linear Gaussian
measurements, Atanasov et al.~\cite{atanasov2014information} formulate the informative path
planning problem as an optimal control problem that affords an offline
solution.
Similarly, Lim et al.~\cite{lim2016adaptive} propose a recursive divide and conquer approach to active information gathering for discrete hypotheses, which is near-optimal in the noiseless case.

Planning for  information maximization-based methods
typically proceeds with a receding horizon \cite{bai2016information,
  marchant2012bayesian, marchant2014bayesian, martinez2009bayesian,
  guestrin2005near}.  For example, Ristic et al. \cite{ristic2010information} formulate information
gathering as a partially observable Markov decision process and approximate a solution using a receding horizon.
Marchant et al. \cite{marchant2012bayesian} combine upper confidence bounds
(UCBs) at potential source
locations with a penalization term for travel distance to define a greedy acquisition function for Bayesian optimization. Their subsequent work %Rather thanoptimize individual steps,
\cite{marchant2014bayesian} reasons at the path level
to find longer, more informative trajectories. Noting the limitations of a
greedy receding horizon approach, \cite{hitz2014fully} incentivizes exploration
by using a look-ahead step in planning. Though similar in spirit to these
information seeking approaches, a key benefit of $\snakeucb$ is that it
is not greedy, but rather iterates over a global path.

% \textbf{Gaussian processes for information maximization.}
Information maximization methods typically
require a prior distribution on the underlying signals. Many
%prior
 active sensing approaches model this prior as being drawn from a Gaussian process (GP) over
an underlying space of possible functions
\cite{miller2016ergodic,bai2016information,marchant2012bayesian},
tacitly enforcing the assumption that the sensed signal is
smooth~\cite{marchant2012bayesian}. In certain applications, this is well
motivated by physical laws, e.g. diffusion~\cite{hitz2014fully}. However, GP
priors may not reflect the sparse, heterogeneous emissions encountered in
radiation detection and similar problem settings. 
%Hence, $\snakeucb$ does not rely upon GP priors.

% \textbf{Multi-armed bandit methods.}
\subsection{Multi-armed bandit methods}
$\snakeucb$ draws heavily on confidence-bound based algorithms
from the pure exploration bandit literature \cite{even2006action,
  audibert2010best, jamieson2014lil}. %We formulate signal strength as a
%``reward'' emanating from a set of possible source locations.
In contrast to these works, our method explicitly incorporates a physical
sensor model and
allows for efficient measurement allocation despite the physical movement constraints inherent to mobile robotic
sensing. Other works have studied spatial constraints in the online,
``adversarial'' reward setting \cite{koren2017multi, bubeck2017kserver}.
Baykal et al.~\cite{baykal2017persistent} consider spatial constraints in a
persistent surveillance problem, in which the objective is to observe as many
events of interest as possible despite unknown, time-varying event statistics.
Recently, Ma et al. \cite{ma2017active}
encode a notion of spatial hierarchy in designing informative trajectories,
based on a multi-armed bandit formulation. While \cite{ma2017active} and
$\snakeucb$ are similarly motivated, %hierarchical planning is sub-optimal for
% the radiation sensing problem due to the rapid decay of sensitivity with
% distance to the signal source.
hierarchical planning can be inefficient for many sensing models, e.g.
for short-range sensors, or signals that decay quickly with distance from the source.
%when sensitivity decays with squared distance to the source of a signal, as in Sec.~\ref{sec:physical_ci}.
%In the different setting of
%online, adversarial bandit reward maximization, the problem of spatial
%constraints has been studied from a worst-case perspective in topologies of
%theoretical interest%.

Bandit algorithms are also studied from a Bayesian perspective, where a prior is placed over underlying rewards.
 For example, Srinivas et al. \cite{srinivas2012information} %provide a theoretical analysis of their GP-UCB algorithm for the setting where rewards are drawn from a GP prior,
provide an interpretation of the GP upper confidence bound (GP-UCB) algorithm in terms of information maximization.
$\snakeucb$ does not use such a prior, and is more
similar to the lower and upper confidence bound (LUCB)
algorithm~\cite{kalyanakrishnan2012pac}, but opts for successive elimination
over the more aggressive LUCB sampling strategy for measurement allocation.

A multi-armed bandit approach to active exploration in Markov decision processes (MDPs) with transition costs is studied in \cite{Tarbouriech2019ActiveEI}, which details trade-offs between policy mixing and learning environment parameters. This work highlights the potential difficulties of applying a multi-armed bandit approach while simultaneously learning robot policies. In contrast, we show that decoupling the use of active learning during the sampling decisions from a fixed global movement path confers efficiency gains under reasonable environmental models.

%In addition, \cite{srinivas2012information} provide an analysis of bandit problems in Bayesian setting.

%Whereas the above works study bandits in a non-Bayesian setting, one can also establish an equivalence between Bayesian bandit problems and information

%A theoretical discussion of the connection of multi-armed bandits and GPs is presented in \cite{srinivas2012information}, where the
%GP-UCB algorithm is used to estimate the function uniformly by assigning each sampled
%point to its mean measured emission. \esther{@max take a look at whether this makes sense here.}

% \esther{!just moved this here; needs attention!} $\snakeucb$ extends beyond
% tradional pure exploration approaches \esther{(cite)} from the multi-armed
% bandits literature on two counts. First, it combines principled measurement
% allocation with continuous trajectory planning for an embodied agent. Second,
% we show how to develop actionable confidence intervals for localized signal
% estimation, even when observed measurements are additive, weighted mixtures of
% point emissions. $\snakeucb$ empirically outperforms both an information
% maximization and a uniform-sampling baseline on a simulated radiation
% detection task, and is also demonstrated in hardware (Sec.~\ref{sec:results}).

% \textbf{Other source-seeking methods.}
\subsection{Other source seeking methods}
Other notable extremum seeking methods include those that emulate gradient ascent in the physical domain \cite{biyik2008gradient, matveev2016extremum,
  porat1996localizing}, take into account specific environment signal characteristics~\cite{khodayi2019model}, or are specialized for particular vehical dynamics~\cite{mellucci2016source}.
Modeling emissions as a continuous
field, gradient-based approaches estimate and follow the gradient of the
measured signal
toward local maxima \cite{biyik2008gradient, matveev2016extremum,
  porat1996localizing}. One of the key drawbacks of gradient-based methods is
their susceptibility to finding local, rather than global, extrema. Moreover,
the error margin on the noise of gradient estimators for large-gain sensors measuring noisy signals can
be prohibitively large \cite{vasiljevic2008error}, as is the case in RSS.  %as is the
%case for Poisson measurement model used to describe gamma radiation detection in Sec.~\ref{sec:methods}.
Khodayi-mehr et al.~\cite{khodayi2019model} handle noisy measurements  by combining domain, model, and parameter reduction methods to actively identify sources in steady state advection-diffusion transport system problems such as chemical plume tracing. Their approach combines optimizing an information theoretic quantity based on these approximations with path planning in a feedback loop, specifically incorporating the physics of advection-diffusion problems. In comparison, we consider planning under specific sensor models, and plan motion path and optimal measurement allocation separately.

%% file: algorithm_snake.tex
% !TEX root = main_ieee.tex

\section{\snakeucbtext Planning Strategy}
\label{sec:algorithm_snake}

%\textbf{Problem statement.}
\subsection{Problem statement}
We
consider signals (e.g. radiation) which emanate from a finite set of environment points $\calS$.
Each point
$x \in \calS$ emits signals
$\{\signal_t(x)\}$ indexed by time $t$ with means $\mu(x)$, independent and identically distributed over time. Our aim is to
correctly and exactly discern the set of the
$k$ points in the environment that emit the maximal signals: %, \esther{TODO @max
                                %make this pttier}
\begin{align}
\label{eqn:s_star}
\Sstk =  \argmax_{\substack{S' \subseteq \calS, |S'| = k}} \sum_{x \in S'}{\mu(x)}% = \argmax_{\substack{S \subset \calS \\ |S| = k}} \sum_{x \in S}{\expect{\signal_t(x)}}.
\end{align}
for a pre-specified integer $ 1 \leq k \leq |\calS|$.
%a priori.
%Note that the maximum of the summation in \eqref{eqn:s_star} is satisfied when $S$ contains the $k$ strongest signal emitters in the enviroment.
Throughout, we assume that the set of maximal emitters $\Sstk$ is unique.

In order to decide which points are maximal emitters, the robot takes sensor
measurements along a fixed path $\trajpoints = (z_1,\dots,z_n)$ in the robot's configuration space.
Measurements are determined by a known \emph{sensor model}
$h(x,z)$ that describes the contribution of environment point
$x~\in~\calS$ to a sensor measurement collected from sensing configuration $z \in \config$.
We consider a linear sensing model in which the total observed measurement at
time $t$, $\signaly_t(z)$, taken from sensing configuration $z$, is the weighted sum of the contributions $\{\signal_t(x)\}$ from all environment points:
\begin{align}
\signaly_t(z) = \sum_{x \in \calS} h(x, z) \signal_{t}(x)
\label{eq:signal_model}
\end{align}
Note that while $h(x,z)$ is known, the $\{\signal_t\}$ are unknown and must be estimated via the observations $\{\signaly_t\}$.

The path of sensing configurations, $\config$, should be as short as possible,
while providing sufficient information about the entire environment.
%after taking measurements at each configuration $z \in \config$.
%should be designed so that after taking measurements at each
%configuration in $z \in \config$, we obtain sufficient information about each point in
%the environment.
\iftoggle{ieee}
{
}
{
  This may be expressed as a condition on the minimum aggregate
sensitivity $\alpha$ to any given environment point $x$ over the sensing path $\config$:
\begin{align}
\label{eqn:min_sensitity}
\sum_{z \in \config} h(x,z) \geq \alpha > 0 \quad\forall x \in \calS
\end{align}
}
Moreover, we need to disambiguate between contributions from different environment points $x,x' \in \calS$.
We define the matrix $H \in \RR^{|\calS| \times |\config|}$ that encodes the sensitivity of each sensing configuration $z_j \in \config$ to each point $x_i \in \calS$, so that
$H_{ij} := h(x_i,z_j)$.
Disambiguation then translates to a rank constraint $\textrm{rank}(H) \geq
|\calS|$, enforcing invertibility of $H H^T$. Sections~\ref{sec:ptwise_ci} and~\ref{sec:physical_ci} define two specific sensitivity functions that we consider in the context of the RSS % \emph{radioactive source-seeking} problem, or $\RSS$0
  problem.
\iftoggle{ieee}{}{
  In Section~\ref{sec:generalizations}, we discuss sensitivity functions that may arise in other application domains.
}

%\maxs{@note: other future things/experiment design put in section Generalizations}

%For our theoretical analysis, we will consider a simplified mathematical model for \emph{pointwise sensing}; that is, $\config = \calS$, each configuration $z$ corresponds to a position $x$ in the grid, and the sensitive function is given by $h(x,z) = \I(x = z)$. We further consider a simplified dynamics model, \maxs{explain further?} where units of time are normalized so that, at top speed, the drones can traverse one grid point in one time step.
%In our experiments, we will consider \maxs{@dfk/esther ...} \maxs{moreover acceleration and corners are things}

%\esther{for us it it pointgrid, distuss more complicated options/experiment design.}

%The goal is to plan a sequence of configurations $z \in \config$, along with timing patterns, which together specify a sensing trajectory which allows the agent to disambiguate $S^*(k)$ from the rest of the environment points.

%\esther{It is important here to note that the signals themselves are actually RVS, coming from some process e.g. poisson counts or mean + gaussian noise. Should really address this in the intro also @max can you take a crack at adding this?}

\input{algorithm_block}

\begin{figure}[!t] %  figure placement: here, top, bottom, or page
   \centering
   \iftoggle{ieee}{
   \includegraphics[width=\linewidth]{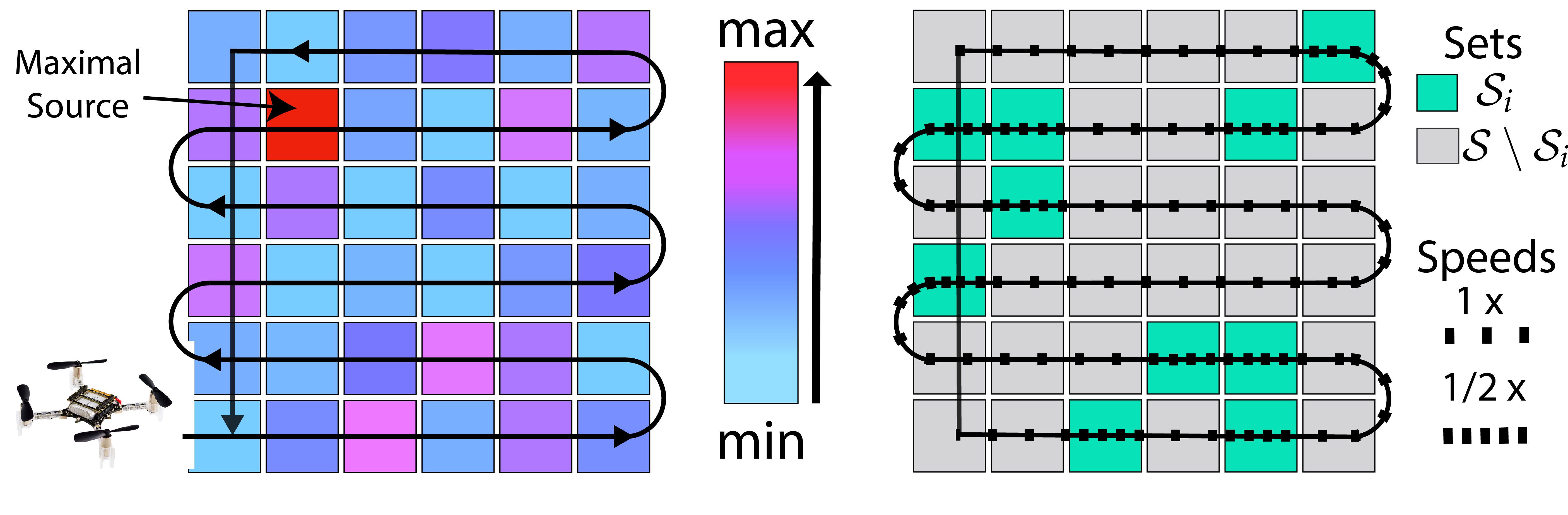}}
   {\includegraphics[width=.8\linewidth]{figures_tro/example_alg_run.png}}
  \caption{(left) Raster path $\config$ over an example grid environment of size $6 \times 6$. The path ensures that each point is sufficiently measured during each round. (right) Illustrative trajectory for round $i=1$.
    Dots indicate %approximate
     measurements. $\snakeucb$ slows to take twice as many measurements over points $x \in \calS_i$.}
    % compared with elsewhere.}
%\vspace{-1.5em}
\label{fig:snake_pattern}
\end{figure}
%\textbf{The AdaSearch algorithm.}
\subsection{The $\snakeucb$ algorithm}
$\snakeucb$ (Alg.~\ref{snake_lucb++})
concentrates measurements in regions of uncertainty until we are confident about which points belong to $S^*(k)$.
At each round $i$, we maintain a set of environment points $\Stop_i$ that we are confident
are among the top-$k$, and a set of candidate points $\calSi$ about which we are
still uncertain.
As the robot traverses the environment, new sensor measurements allow us to update
the lower and upper confidence bounds
\iftoggle{ieee}{$\left[\LCB_i(x) , \UCB_i(x) \right]$ for the mean signal $\mu(x)$ of each $x \in \calS_i$
}{
\begin{align*}
\left[\LCB_i(x) , \UCB_i(x) \right] \text{ for each } x \in \calS_i.
\end{align*}
} and prune the uncertainty set $\calS_i$.
The procedure for constructing these intervals from observations should ensure that for every $x \in \calS_i$,
$\LCB_i(x) \leq\mu(x)\leq\UCB_i(x)$ with high probability. Sections~\ref{sec:ptwise_ci} and~\ref{sec:physical_ci} detail the definition of these confidence intervals under different sensing models.
%The confidence intervals depend on samples collected over round $i$.

Using the updated confidence intervals, we expand the set $\Stop_{i+1}$ and prune the set $\calS_{i+1}$. We add to the top-set $\Stop$ all points $x \in \calS_i$ whose lower confidence bounds exceed the upper confidence bounds of all but $(k - |\Stop_i|)$ points in $\calS_i$; formally,
\iftoggle{ieee}
{
  \begin{multline}\label{eq:top_elim}
    \Stop_{i+1} \leftarrow \Stop_i \cup \{x \in \calS_i ~|~ \LCB_{i}(x) > \\
    (k - |\Stop_i| + 1)\text{-th largest } \UCB_{i}(x'), x' \in \calS_i \}.
  \end{multline}
}
{
  \begin{align}\label{eq:top_elim}
    \Stop_{i+1} &\leftarrow \Stop_i \cup \{x \in \calS_i ~|~ \LCB_{i}(x)
    > (k - |\Stop_i| + 1)\text{-th largest } \UCB_{i}(x'), x' \in \calS_i \}.
  \end{align}
}
In Eq.~\eqref{eq:top_elim}, we need not re-evaluate confidence intervals for points $x$ already in $\Stop_i$ when producing the set $\Stop_{i+1}$, and can only consider new points. This is explained in the proof of correctness (\Cref{lem:main_correctness} and \Cref{thm:snakeucb_runtime_topk_inmain}), where we show that with high probability, points $x \notin \Sstk$ are never incorrectly added to the estimate of the top set $\Stop_i$. 

Next, the points added to $\Stop_{i+1}$ are removed from $\calS_{i+1}$, since we are now certain about them.
Additionally, we remove all points in $\calS_i$ whose upper confidence bound is lower that than the lower confidence bounds of at least $k - |\Stop_{i+1}|$ points in $\calS_i$. The set $\calS_{i+1}$ is defined constructively as:
  \iftoggle{ieee}
  {
    \begin{multline}
    \label{eq:bot_elim}
    \calS_{i+1} \leftarrow \{x \in \calS_i  ~|~ x \notin \Stop_{i+1} \text{ and } \UCB_{i}(x) \ge \\ (k - |\Stop_{i+1}|)\text{-th largest }  \LCB_{i}(x'), x' \in \calS_i\}.
  \end{multline}
  }
  {
    \begin{align}
    \label{eq:bot_elim}
    \calS_{i+1} &\leftarrow \{x \in \calS_i  ~|~ x \notin \Stop_{i+1} \text{ and } \UCB_{i}(x) \ge (k - |\Stop_{i+1}|)\text{-th largest }  \LCB_{i}(x'), x' \in \calS_i\}.
  \end{align}
  }

  % \begin{align}
  %   \label{eq:bot_elim}
  %   \calS_{i+1} &\leftarrow \{x \in \widetilde{\calS}_i  ~|~  \UCB_{i+1}(x) < (k - |\Stop_{i+1}|)\text{-th largest }  \LCB_{i+1}(x'), x' \in \calS_i\},\\
  %    \text{where}&~\widetilde{\calS}_i := \calS_i \setminus \Stop_{i+1}.\nonumber.
  % \end{align}

%\textbf{Trajectory planning for $\snakeucb$.}
\subsection{Trajectory planning for $\snakeucb$}
The update rules~\eqref{eq:top_elim} and \eqref{eq:bot_elim} only depend on confidence intervals for points $x
\in \calS_i$. At each round, $\snakeucb$ chooses a subset of the sensing configurations
$\config_i \subseteq \config$ which are informative to disambiguating the points
remaining in $\calS_i$.
%The informative
%configuration subset $\config_i$ could be chosen, for example, to satisfy
%the constraint that $\textrm{rank}(H) \ge \cardS$ is full rank. %, replacing $\calS$ with $\calS_i$.
% For omni-directional sensors, choosing $\config$ and $\config_i$ is relatively straightforward (see Sec. \ref{sec:snake_path}). We discuss generalizing to more sophisticated sensors
% %as well as application-specific experiment design objectives
% in Sec.~\ref{sec:generalizations}.

$\snakeucb$ defines a trajectory by following the fixed path $\trajpoints$ over all configurations,
slowing down to spend time $2^i \tau_0$ at
 informative configurations in $\config_i$, and spending minimal time $\tau_0$ at all other configurations in $\config \setminus \config_i$.
% Specifically, at the start of round $i$, $\snakeucb$ plans its trajectory as follows:
% \begin{itemize}
% 	\item[(a)] $\snakeucb$ traverses the $z \in \config$ in the order specified by $\trajpoints$.
% 	\item[(b)] $\snakeucb$ spends time at least $\tau_i := 2^i \tau_0$ at each informative location $z \in \config_i$, and $\tau_0$ time steps at each uninformative location $z \in \config \setminus \config_i$.
% 	% \item[(b)] $\snakeucb$ starts at $\barz^{(1)} \in \trajpoints$, an only visits the $j$-th location $\barz^{(j)}$ in $\trajpoints$ after visiting the $j-1$-st location, $\barz^{(j-1)}$
% \end{itemize}
 %and chooses its locations $z_t.$
Doubling the time spent at each $z \in \config_i$ in each round
amortizes the time spent traversing the entire path $\config$. 
For omnidirectional sensors, a
simple raster pattern (Fig.~\ref{fig:snake_pattern}) suffices for $\trajpoints$ and choosing $\config_i$ is relatively straightforward (see Sec. \ref{sec:snake_path} for further discussion on choosing $\config$ and $\config_i$). Finally, we remark that one can set the time per configuration $z \in \config_i$ as $\tau_i = \tau_0 c^i$ for any constant $c > 1$; this yields similar theoretical guarantees, and constants other than $c=2$ may confer slight benefits in practice.
%More complex cases are discussed in Sec.~\ref{sec:generalizations}.

% \textbf{Visiting only $\config_i$.}
%\subsection{Optimizing informative paths}
We could also design a
trajectory that visits the $z \in \config_i$ and minimizes total travel
distance each round, e.g. by approximating a traveling salesman solution. In
practice, this would improve upon the runtime of the fixed raster path
suggested above. In this work, we use a raster pattern 
to emphasize the gains due to our main algorithmic components: global coverage and adaptive measurement
allocation.

%\textbf{Correctness.}
\subsection{Correctness}
Lemma~\ref{lem:main_correctness} establishes that the two update rules above guarantee the overall correctness of $\snakeucb$, whenever the confidence intervals $[\LCB_{j}(x),\UCB_{j}(x)]$ actually contain the correct mean $\mu(x)$:
\begin{restatable}[Sufficient Condition for Correctness]{lemma}{sufficient
    condition for correctness}\label{lem:main_correctness} For each round $i \ge 0$, $\Stop_{i} \cap \calS_i = \emptyset$.
    \iftoggle{ieee}
    {
    Moreover, whenever the confidence intervals satisfy the coverage property:
       \begin{align}\label{eq:correct_coverage}
      \forall j \le i, x\in\calS_j, \quad \mu(x) \in [\LCB_{j}(x),\UCB_{j}(x)],
      \end{align}
    }
    {
    Moreover, whenever the confidence intervals satisfy the coverage property:
      \begin{align}\label{eq:correct_coverage}
      \text{for all rounds }j \le i~ \text{ and all }x\in\calS_j, \quad \mu(x) \in [\LCB_{j}(x),\UCB_{j}(x)],
      \end{align}
    }
then $\Stop_{i+1} \subseteq \Sstk \subseteq \{\Stop_{i+1} \cup \calS_{i+1}\}$.
If~\eqref{eq:correct_coverage} holds for all rounds $i$, then $\snakeucb$ terminates
and correctly returns $\Sstk$.
\end{restatable}

\iftoggle{ieee}{
\begin{IEEEproof}(Sketch)
Non-intersection of $\Stop_i$ and $\calS_i$ follows inductively from update rule~\eqref{eq:bot_elim} and the initialization $\Stop_{0} = \emptyset, \calS_0 = \calS$.

The overlapping set property  $\Stop_{i+1} \subseteq \Sstk \subseteq \{ \Stop_{i+1} \cup \calS_{i+1} \}$ follows by induction on the round number $i$.
When $i=0$, $\Stop_{0} = \emptyset \subseteq \Sstk \subseteq \calS = \{ \Stop_0 \cup \calS_0\}$.
Now assume that $\Stop_{i} \subseteq \Sstk \subseteq \{ \Stop_{i} \cup \calS_{i}\}$ holds for round $i$.
Update rule~\eqref{eq:top_elim} moves a point $x$ from $\calS_i$ to $\Stop_{i+1}$ only if its LCB is above the $(k+1)$-th largest UCB of all points in $\calS_i \cup \Stop_i$.
By~\eqref{eq:correct_coverage}, $\LCB_i(x) \leq \mu(x) \leq \UCB_i(x)~\forall x \in \calS_i$, so that $\mu(x)$ must be greater than or equal to the $(k+1)$-st largest means of the points in $\calS_i \cup \Stop_i$. Therefore, this $x$ must belong to $\Sstk$, establishing that $\Stop_{i+1} \subseteq \Sstk$.
Similarly, by update rule~\eqref{eq:bot_elim}, a point $x'$ is only removed from $\Stop_{i} \cup \calS_{i}$ if its UCB is below the $k$ largest LCBs of points in $\calS$, such that $\mu(x')$ is less than or equal to at least $k$ other means. Thus, such a point $x'$ cannot be in $\Sstk$. This establishes that $\Sstk \subseteq \{\Stop_{i+1} \cup \calS_{i+1}\}$.

Finally, at termination we have $\calS_{\ifin} = \emptyset$, so that $\Stop_{\ifin} \subseteq \Sstk \subseteq \{\Stop_{\ifin} \cup \calS_{\ifin}\} = \Stop_{\ifin}$, so that $\Sstk = \Stop_{\ifin}$.
\end{IEEEproof}}
{ % arxiv version
  \begin{proof}(Sketch)
\footnote{For an extended version of the proofs presented in this paper, see the appendix.}
%}
Non-intersection of $\Stop_i$ and $\calS_i$ follows inductively from update rule~\eqref{eq:bot_elim} and the initialization $\Stop_{0} = \emptyset, \calS_0 = \calS$.

The overlapping set property  $\Stop_{i+1} \subseteq \Sstk \subseteq \{ \Stop_{i+1} \cup \calS_{i+1} \}$ follows by induction on the round number $i$.
When $i=0$, $\Stop_{0} = \emptyset \subseteq \Sstk \subseteq \calS = \{ \Stop_0 \cup \calS_0\}$.
Now assume that $\Stop_{i} \subseteq \Sstk \subseteq \{ \Stop_{i} \cup \calS_{i}\}$ holds for round $i$.
Update rule~\eqref{eq:top_elim} moves a point $x$ from $\calS_i$ to $\Stop_{i+1}$ only if its LCB is above the $(k+1)$-th largest UCB of all points in $\calS_i \cup \Stop_i$.
By~\eqref{eq:correct_coverage}, $\LCB_i(x) \leq \mu(x) \leq \UCB_i(x)~\forall x \in \calS_i$, so that $\mu(x)$ must be greater than or equal to the $(k+1)$-st largest means of the points in $\calS_i \cup \Stop_i$. Therefore, this $x$ must belong to $\Sstk$, establishing that $\Stop_{i+1} \subseteq \Sstk$.
Similarly, by update rule~\eqref{eq:bot_elim}, a point $x'$ is only removed from $\Stop_{i} \cup \calS_{i}$ if its UCB is below the $k$ largest LCBs of points in $\calS$, such that $\mu(x')$ is less than or equal to at least $k$ other means. Thus, such a point $x'$ cannot be in $\Sstk$. This establishes that $\Sstk \subseteq \{\Stop_{i+1} \cup \calS_{i+1}\}$.

Finally, at termination we have $\calS_{\ifin} = \emptyset$, so that $\Stop_{\ifin} \subseteq \Sstk \subseteq \{\Stop_{\ifin} \cup \calS_{\ifin}\} = \Stop_{\ifin}$, so that $\Sstk = \Stop_{\ifin}$.
\end{proof}

}

%arXiv proof is in appendix

Lemma \ref{lem:main_correctness} provides a backbone upon which we
construct a probabilistic correctness guarantee in Sec.~\ref{sec:analysis}.
If the event~\eqref{eq:correct_coverage} holds  over all rounds with some probability $1 -
\deltot$, then
$\snakeucb$ returns the correct set $\Sstk$ with the same probability $1 -
\deltot$.

%% file: algorithm_block.tex
% !TEX root = main_ieee.tex

\begin{algorithm}[t]
  \textbf{Input} Candidate points of interest $\calS$;
  sensing path of configurations $\config$;
  number of points of interest $k$;
  minimum measurement duration $\tau_0$;
  procedure for constructing $[\LCB_i(x), \UCB_i(x)]$ %for all rounds $i \ge 0$ and $x \in \calS_i$
  (e.g., as in Sections~\ref{sec:ptwise_ci} and~\ref{sec:physical_ci});
  Confidence parameter $\deltot$. \\ %sensitivity function $h(\cdot, \cdot)$,
  \textbf{Initialize } $\Stop_0 = \emptyset, \calS_0 = \calS$ \\
  % \textbf{If} stopping criterion satisfied, \textbf{Break} \\
  \textbf{For} rounds $i = 0, 1, 2, \dots$\\
  \Indp
  % \textbf{If } $|\Stop_i| \ge k$: \label{alg:termination_criterion}\\
  % \Indp
  % \textbf{Return} any $k$ elements of $\Stop_i$\\
  % \Indm
  % \textbf{If } $|\Stop_i| + |\calS_i| \le k$: \label{alg:termination_criterion}\\
  \textbf{If } $\calS_i = \emptyset$, \textbf{Return} $\Stop_i$ \label{alg:termination_criterion}\\
  %\esther{@max I know this isnt right but its much clearer whats going on \\}
  % and its estimated emission rate $\muhat(x)$. \\
  \textbf{Choose} configuration subset $\config_i \subseteq \config$ that is informative about environment points $x \in \calS_i$. \label{algline:design}\\
  \textbf{Execute} a trajectory along path $\trajpoints$ that spends time $\tau_i = \tau_{0} \cdot 2^i$ at each $z \in \config_i$ and time $\tau_0$ at each $z \in \config \setminus \config_i$. Meanwhile, observe signal measurements according to \eqref{eq:signal_model}. \label{algline:trajectory} \\
  %\textbf{Collect} measurements by following the planned trajectory in
  %Line~\ref{algline:trajectory}.\label{alg:collection}\\
  \textbf{Update} $\left[ \LCB_{i}(x),\UCB_{i}(x) \right]$ for all $x \in \calS$. \label{alg:update_step} \\
  \textbf{Update} Augment $\Stop_i$ according to~\eqref{eq:top_elim}, and prune $\calS_i$ according to~\eqref{eq:bot_elim}.

  \caption{$\snakeucb$} \label{snake_lucb++}
\end{algorithm}

%% file: radioactive_source_seeking.tex
% !TEX root = main_ieee.tex
\section{Radioactive Source-Seeking with Poisson Emissions}
\label{sec:methods}

While $\snakeucb$ applies to a range of adaptive sensing problems,
for concreteness we now refine our focus to the problem of radioactive
source-seeking ($\RSS$)  with an omnidirectional %, perfectly absorbing
sensor. % onboard a quadrotor helicopter.
The environment is defined by
potential emitter locations which lie on the ground plane, i.e. $x \in \calS
\subset \R^2 \times \{0\}$, and sensing configurations encode spatial position,
i.e. $z \in \config \subset \R^3$. Environment points
emit gamma rays according to a Poisson process, i.e. $\signal_t(x) \sim
\Poi\left(\mu(x)\right)$. Here, $\mu(x)$ corresponds to rate or intensity of emissions from point $x$.

Thus, the number of gamma rays observed over a time interval of
length $\tau$ from configuration $z$ has distribution
\begin{eqnarray}\label{eq:Poisson_sensing}
   \signaly_t(z) \sim \Poi\Big( \tau \cdot \sum_{x \in \calS} h(x, z) \mu(x)\Big)~,
\end{eqnarray}
%\vspace{-.25em}
\noindent where  $h(x, z)$ is specified by the sensing model.
In the following sections, we introduce two sensing models: a pointwise sensing
model amenable to theoretical analysis (Sec.~\ref{sec:ptwise_ci}), and a more
physically realistic sensing model for experiments (Sec.~\ref{sec:physical_ci}).

In both settings, we develop appropriate confidence intervals for use in the
$\snakeucb$ algorithm. We introduce the specific path used for global
trajectory planning in Sec.~\ref{sec:snake_path}.
Finally, we conclude with two benchmark algorithms to which we compare
$\snakeucb$ (Sec.~\ref{sec:baseline}).

\subsection{Pointwise sensing model\label{sec:ptwise_ci}}

First, we consider a simplified sensing model, where the set of sensing locations
$\config$ coincides with the set $\calS$ of all emitters, i.e. each $z \in
\config$ corresponds to exactly one $x \in \calS$ and vice versa. The sensitivity function is defined as
${
h(x,z) := \I(x = z)} = [1 \textrm{ if } x = z, 0 \textrm{ if } x \not= z]$.

Now we derive confidence intervals for Poisson counts observed according to this sensing model.
Define $\cnt(x)$ to be the total number
of gamma rays observed during the time interval of length $\tau$ spent at $x$.
The maximum
likelihood estimator (MLE) of the emission rate for point $x$ is $\muhat(x) =
\frac{\cnt(x)}{\tau}$. \iftoggle{ieee}{Using standard bounds for Poisson tails \cite{boucheron2013concentration}, we introduce}{In Appendix~\ref{sec:main_theory_results}, we introduce the}
\textit{bounding functions} $U_-(\cdot,\cdot)$ and $U_+(\cdot,\cdot)$:
 \begin{align*}
    U_+\left(\cnt,\delta\right)&:= 2 \log(1/\delta) + \cnt + \sqrt{2\cnt \log\left(1/\delta\right)} ~~\mathrm{and} \\
    U_-\left(\cnt,\delta\right) &:= \max\left\{0,\cnt  - \sqrt{2\cnt\log(1/\delta)}\right\}
  \end{align*}
Then for any $\lambda \ge 0$, $\cnt \sim \Poi(\lambda)$, and $\delta \in (0,1)$,
\begin{align*}
\Pr[U_-(\cnt,\delta) \le \lambda \le
  U_+(\cnt,\delta)] \ge 1 -2\delta.
 \end{align*}
Let $\cnt_i(x)$ denote the number of gammas rays observed from emitter $x$ during round $i$, so that  ${\cnt(x) \sim \Poi(\tau_i \mu(x))}$.
For any point $x \in \calS_i$, the corresponding duration of measurement would be $\tau_i$.
The bounding functions above provide the desired
confidence intervals for signals $\mu(x), \forall x \in \calS_i$:
\iftoggle{ieee}
{
%   \begin{align}
%   \LCB_i(x) &:= \frac{1}{\tau_i}U_-\left(\cnt_i(x), \delta_i\right)~ \nonumber\\
%   \UCB_i(x) &:=  \frac{1}{\tau_i} U_+\left(\cnt_i(x),\delta_i \right)~,
%   \label{eqn:lcb_ucb}
% \end{align}
$\left[\LCB_i(x),\UCB_i(x)\right] := \frac{1}{\tau_i}\left[U_-\left(\cnt_i(x), \delta_i\right),
   U_+\left(\cnt_i(x),\delta_i \right)\right]$.
}
{
  \begin{eqnarray}
  \LCB_i(x) := \frac{1}{\tau_i}U_-\left(\cnt_i(x), \delta_i\right)~,\quad~
  \UCB_i(x) :=  \tfrac{1}{\tau_i} U_+\left(\cnt_i(x),\delta_i \right)~,
  \label{eqn:lcb_ucb}
\end{eqnarray}
}
%where $\delta_i := \epsilon/{(4|\calS| i^2)}$.  T
This bound implies that the inequality ${\tau_i
  \LCB_i(x) \le \tau_i \mu(x) \le \tau_i \UCB_i(x)}$ holds with probability
$1 - 2\delta_i$.
% = {1~-~\delta/\left(2|\calS| i^2\right)}$.
Dividing by $\tau_i$, we see that
$\LCB_i(x)$ and $\UCB_i(x)$ are valid confidence bounds for
$\mu(x)$.

The term $\delta_i$ can be thought of as an ``effective confidence''
for each interval that we construct during round $i$.
In order to achieve the correctness in Lemma~\ref{lem:main_correctness} with
overall probability $1-\deltot$, we set the effective confidence $\delta_i$ at each round to be $\delta_i = \deltot/{(4|\calS| i^2)}$\iftoggle{ieee}{}{(see Appendix~\iftoggle{ieee}{\zref{sec:egood:proof}}{\ref{sec:egood:proof}})}.

\subsection{Physical sensing model\label{sec:physical_ci}}
A more physically accurate sensing model for RSS reflects that %, in general,
the gamma ray counts at each location are sensitivity-weighted combinations of emissions from each environment point.
Conservation of energy in free space allows us to approximate the sensitivity with an inverse-square law $h(x,z) := c/ \|x - z\|_2^2$, with $c$ a known, sensor-dependent constant.
More sophisticated approximations are also possible \cite{ristic2010information}.

Because multiple environment points $x$ contribute to the counts observed from
any  sensor position $z$, the MLE $\hat \mu$
for the emission rates at all $x \in \calS$ is difficult to compute efficiently. However, we can
approximate it in the limit: $ \frac{1}{\sqrt{\tau}}\Poi(\tau\mu)
\overset{d}{\rightarrow} \normal{\mu, \mu} $ as $\tau \rightarrow \infty$. Thus, we may compute $\hat{\mu}$ as the least squares solution:
\begin{equation}
  \label{eqn:lstsq}
  \hat{\mu} = \arg\min_{\vec{\mu}} \| \tilde H^T \vec{\mu} - \vec{\signaly} \|_2^2~,
\end{equation}
where $\vec{\mu} \in \RR^{|\calS|}$ is a vector representing the mean emissions from each $x \in \calS$,  $\vec{\signaly} \in \RR^m$ is a vector representing the observed number of counts at each of $m$ consecutive time intervals, and $\tilde H \in \RR^{|\calS| \times m}$ is a
rescaled sensitivity matrix such that $\tilde H_{ij}$ gives the measurement-adjusted sensitivity of the $i^{th}$ environment point to the sensor at the $j^{th}$ sensing position.\footnote{Specifically, we define
$\tilde H_{ij} = h(x_i, z_j) / (\signaly_j + b)$.
The rescaling term $\signaly_j + b$ is a plug-in
estimator for the variance of $\signaly_j$ (with small bias $b$ introduced
for numerical stability), which down weights  higher variance measurements.}
 The resulting confidence bounds are given by the
standard Gaussian confidence bounds:
\iftoggle{ieee}
{
  \begin{equation}
  \label{eqn:cis}
  [\LCB_i(x_k), \UCB_i(x_k)] := \hat{\mu}(x_k) \pm \alpha(\delta_i) \cdot\Sigma^{1/2}_{kk}~,
\end{equation}
where  $\Sigma := (\tilde H \tilde H^T)^{-1}$, and
}
{
  \begin{equation}
  \label{eqn:cis}
  [\LCB_i(x_k), \UCB_i(x_k)] := \hat{\mu}(x_k) \pm \alpha(\delta_i) \cdot\Sigma^{1/2}_{kk}~ \quad\textnormal{where\ } \Sigma := (\tilde H \tilde H^T)^{-1} \enspace
\end{equation}
}
$\alpha(\delta_i)$ controls the round-wise effective confidence widths in equation \eqref{eqn:cis} as a function of the desired threshold probability of overall error, $\deltot$.
%We remark that $\Sigma$ implicitly captures the effect of the plug-in mean estimate $y_i$ on the variance estimate.
We use a Kalman
filter to solve the least squares problem \eqref{eqn:lstsq}
and compute the confidence intervals \eqref{eqn:cis}.

% Concretely, we approximate each new sensor measurement (photon count) $y \in
% \mbb{N}$ as approximately Gaussian, i.e. $y \sim \normal{\expect{y},
% \expect{y} + b}$.\footnote{We use a plug-in estimator for $\expect{y}$, i.e.
% $\expect{y} \approx y$, and set the small bias term $b = 10$ to improve
% numerical stability.}

\subsection{Design and planning for $\snakeucb$ : choosing $\config$ and $\config_i$}\label{sec:snake_path}

\textbf{Pointwise sensing model.} In the pointwise sensing model, $\config = \calS$ and the most informative sensing locations  $\config_i$ at round $i$ are precisely $\calS_i$.
We therefore choose the path $\trajpoints$ to be a simple space filling curve
over a raster grid, %depicted in Fig.~\ref{fig:snake_pattern},
which provides coverage of all of $\calS$. We
adopt a simple dynamical model of the quadrotor in which it can fly at up to a
pre-specified top speed, and where acceleration and deceleration times are
negligible. This model is suitable for large outdoor environments where
travel times are dominated by movement at maximum speed. We denote this maximum speed as $\tau_0$.
%Thus, the $\snakeucb$ trajectory planned at round $i$ will spend time $\tau_i =
%\tau_0 2^i$ at each $z \in \config_i$, and spends time $\tau_0$ at each $z \in
%\config \setminus \config_i$.
Figure~\ref{fig:snake_pattern} shows
an example environment with raster path $\config$ overlaid (left) and trajectory during round $i=1$
with $\config_1$ shown in teal (right).

\textbf{Physical sensing model.}
Because the physical sensitivity follows an inverse-square law, the most informative measurements about $\mu(x)$
are those taken at locations near to $x$. % (this also $H$ is well-conditioned).
We take measurements at points $z \in \R^3$
two meters above points $x \in \calS$ on the ground plane.
Flying at relatively low height improves the conditioning of the sensitivity matrix~$H$.
%Flying only at very high altitudes may cause $H$ to become ill-conditioned.
We use the same design and planning
strategy as in the pointwise model, following the raster pattern depicted in
Fig.~\ref{fig:snake_pattern}. More generally, one should chose configurations
  $z_i$ from $\config_i$ so that environment point $x_i$ on the ground below
  each is still in
the set of environment points we are unsure about ($\calS_i$ in Eq.~\ref{eq:bot_elim}).

\subsection{Baselines\label{sec:baseline}}
We compare $\snakeucb$ to two baselines: a uniform-sampling based algorithm $\naivesnake$, and a spatially-greedy information maximization algorithm $\infomax$.

\textbf{NaiveSearch algorithm.} As a non-adaptive baseline, we consider a
uniform sampling scheme that follows the raster pattern in
Fig.~\ref{fig:snake_pattern} at constant speed. This global $\naivesnake$ trajectory
results in measurements uniformly spread over the grid, and avoids
redundant movements between sensing locations. The only difference between
$\naivesnake$ and $\snakeucb$ is that $\naivesnake$ flies at a constant speed,
while $\snakeucb$ varies its speed.
% in response to measurements it has
%collected so far.
Comparing to $\naivesnake$ thus separates the advantages of $\snakeucb$'s
adaptive measurement allocation from the effects of its global trajectory heuristic.
Theoretical analysis in Sec.~\ref{sec:analysis} considers a slight variant in which the sampling time is doubled at each round. This doubling has theoretical benefits, but for all experiments we implement the more practical fixed-speed baseline.

\textbf{InfoMax algorithm.} As discussed in Sec. \ref{sec:related_work}, one
of the most successful methods for active search in robotics is receding horizon informative
path planning, e.g. \cite{marchant2014bayesian, martinez2009bayesian}.
% Ideally, such planning schemes lead to trajectories that spend most of their
% time in regions where measurements reveal information about the parameters to
% be estimated. Because measurements collected in higher-emission regions have
% greater variance under the Poisson model, we expect $\infomax$ to spend the
% more of its time in these promising regions. However, the myopic planning
% horizon can cause the planner to be inefficient over long time scales.
% Comparing this class of methods to $\naivesnake$ and $\snakeucb$ helps to
% disambiguate between the effects of receding horizon control (vs. global
% heuristic planning) and those of global uncertainty reduction (vs. successive
% elimination).
We implement $\infomax$, a version of this approach based on
\cite{marchant2014bayesian} and specifically adapted for RSS. %and summarized in Alg. \ref{alg:info_max}.
Each planning invocation solves an information maximization
problem over the space of trajectories ${\xi : [t, t + T_{\text{plan}}] \To
  \mc{B}}$ mapping from time in the next $T_{\text{plan}}$ seconds to a box
$\mc{B} \subset \mbb{R}^3$.
% These trajectories are encoded by second-order Bezier curves, which are
% paramaterized by the current position $p(t)$ of the sensor at time $t$ (which
% is known) and by two control points $p_1, p_2 \in \mc{B}$ for a total of six
% parameters.\footnote{Since the radiation sensor is directionally invariant,
% there is no need to optimize for orientation.} For clarity, we will denote
% these parameters as subscripts, e.g. $\tau_{p(t), p_1, p_2}(\cdot)$. Bezier
% curves provide a relatively low-dimensional representation of smooth
% trajectories, which are easier for a quadrotor to track, and they also provide
% a natural mechanism for encoding a velocity profile along the trajectory.

We measure the information content of a candidate trajectory $\xi$ by
accumulating the sensitivity-weighted variance at each grid point $x \in \calS$
at $N$ evenly-spaced times along $\xi$, i.e.
\begin{equation}
  \label{eqn:info_max}
  \xi^*_t = \arg \max_{\xi}  \sum_{i = 1}^N  \sum_{j = 1}^{|\calS|} \Sigma_{jj} \cdot h\left(x_j,\xi(t + T_{\text{plan}} i / N)\right) \enspace .
\end{equation}
This objective favors taking measurements sensitive to regions with high
uncertainty. As a consequence of the Poisson emissions model, these regions will
also generally have high expected intensity $\mu$; therefore we expect this
algorithm to perform well for the RSS task. We parameterize
trajectories $\xi$ as Bezier curves in $\RR^3$, and use Bayesian
optimization (see \cite{martinez2014bayesopt}) to solve \eqref{eqn:info_max}.
% because of the high computational cost of evaluating the objective
%function.
Empirically, we found that Bayesian optimization outperformed
both naive random search and a finite difference gradient method. We set
$T_{\text{plan}}$ to 30 s and used second-order Bezier curves.
%Longer time
%horizons and higher order Bezier curves quickly become intractable in real time.
% In particular, populating the values of the measurement matrix $\mbf{h}$ at
% each time $\{ t + \frac{i}{N} T_{\text{plan}} \}_{i=1}^N$ requires evaluating
% the sensor's sensitivity to every pixel in the grid.

% \subsubsection{$\naivesnake$}
% \label{subsubsec:uniform_sampling}

% simple scheme for solving our problem is to sample the space uniformly until
% some stopping criterion is satisfied. Since we represent the space as a grid,
% one natural approach to achieving such a uniform sample is to fly in a raster
% (`snake') pattern over the grid at the minimum allowable height.

% \subsection{Stopping Criteria and Metrics}

\textbf{Stopping criteria and metrics.} All three algorithms
use the same stopping criterion, which is satisfied when the $k^{th}$ highest LCB exceeds
the $(k+1)^{th}$ highest UCB.
For $k=1$ emitter, this corresponds to the first round
$i$ in which
$\LCB_i(x) > \UCB_i(x'), \forall x' \in \calS \setminus \{x\}$ for some
environment point $x$.
%\footnote{Confidence bounds over the environment are computed for each algorithm at the same rate, according to the confidence intervals given \eqref{eqn:cis}.}
For sufficiently small probability of error $\deltot$, this
ensures that the top-$k$ sources are almost always correctly identified by all
algorithms.
%/; they are always correctly identified in all experiments in Sec. \ref{sec:results}.

% This gives a global measure of the accuracy of the predicted source emission
% rates in the grid.

% We expect the information gain approach to be superior with respect to this
% metric since it is based on a global information criterion rather, whereas our
% bandit-based algorithm takes measurements only in order to distinguish between
% the two most promising regions.
% \item Euclidean distance of estimated extremal point to true source point, in
%   Euclidean distance. This is a measure of how well we localize the source. We
%   expect both methods to perform well with respect to this metric since they
%   should both find the source eventually.

%   We expect our algorithm to outperform the baseline with respect to this
%   metric since the bandit logic should choose to sample the source pixel most
%   frequently, whereas the baseline information maximization approach will
%   cover the space more evenly.

%   We expect the bandit algorithm to run significantly more quickly than the
%   baseline.

%   \gencomment{We don't have $\alpha$ anymore...}

%% file: theorem_and_baby_proof.tex
%!TEX root = main_ieee.tex
\section{Theoretical Runtime and Sampling Analysis}
\label{sec:analysis}
Separation of sample-based planning and a repeated global trajectory make $\snakeucb$ particularly amenable to runtime and sample complexity analysis. We analyze $\snakeucb$ and $\naivesnake$ under the pointwise sensing model from Sec.~\ref{sec:ptwise_ci}.
%, as they are more amenable
%to such analysis than $\infomax$.
Runtime and sample guarantees are given in Theorem~\ref{thm:snakeucb_runtime_topk_inmain},
with further analysis for a single source in Corollary~\ref{prop:k_equal_one_results} to complement experiments.
 \iftoggle{ieee}{}{Proofs, along with general results for arbitrary $k$, and complimentary lower bounds, are presented in Appendix~\ref{sec:main_theory_results}.}
%Both $\snakeucb$ and $\naivesnake$ correctly identify the top-$k$ sources
%with high probability.
% and Corollary~\ref{prop:k_equal_one_results}sharply quantifies the relative advantage of$\snakeucb$ over $\naivesnake$ in terms of the distribution of the background radiation.
Simulations (Sec.~\ref{sec:results}) show that our theoretical results are indeed predictive of the relative performance of $\snakeucb$ and $\naivesnake$.

We analyze $\snakeucb$ with the trajectory planning strategy outlined in
Sec.~\ref{sec:snake_path}. For $\naivesnake$, the robot spends time $\tau_i$ at
each point in each round $i$ until termination, which is determined by the same
confidence intervals and termination criterion for $\snakeucb$.

We will be concerned with the \emph{total runtime}.
Recall that $\tau_0$ is the time spent over any point when the robot is moving at maximum speed; $\tau_i$ is the time spent sampling candidate points at the slower speed of round $i$.
\begin{align}
\Truntime &= \begin{cases}\sum_{i = 0}^{\ifin} \left(\tau_i|\calS_i| + \tau_0|\calS \setminus \calS_i|\right)  & (\snakeucb)\\
\sum_{i = 0}^{\ifin} \tau_i|\calS| & (\naivesnake)
\end{cases},
\label{eq:runtime_defs}
\end{align}
where $\ifin$ is the round at which the algorithm terminates.
Bounds are stated in terms of divergences between emission rates $\mu_2 \geq \mu_1 > 0$:
\begin{align*}
d(\mu_1, \mu_2) = {(\mu_2 - \mu_1)^2}/{\mu_2}~.
\end{align*}
These divergences approximate the $\KL$-divergence between distributions $\Poi(\mu_1)$
and $\Poi(\mu_2)$\iftoggle{ieee}{}{(see Appendix~\ref{sec:main_theory_results})}, and hence the sample complexity of distinguishing between points emitting photons at rates $\mu_1, \mu_2$.
Analogous divergences are available for any exponential family, for example Gaussian distributions where the divergences are symmetric.% in $\mu_1, \mu_2$.

To achieve the termination criterion (when $\Sstk$ is determined with confidence $1-\deltot$), all points with emission rate below the lowest in $\Sstk$ must be distinguished from $\mu^{(k)}$, the lowest emission rate of points in $\Sstk$. Therefore, for points $x \not\in \Sstk$, we consider divergences $\divergbot$.
%, describing how close $\mu(x)$ is to the mean $\mu^{(k)}$ of the point in $\calS^*(k)$ from which it is hardest to distinguish.
Similarly, all points in $\Sstk$ must be distinguished from the highest background emitter corresponding to the divergences $\divergtop$,
% \begin{align*}
%   \divergtop &= \frac{(\mu(x) - \mu^{(k+1)})^2}{\mu(x)},~(\mu(x) > \mu^{(k+1)})  \quad \text{and } \quad\\%\\\text{ for}~ \mu(x) > \mu^{(k+1)} \quad \text{and}\\
%   \divergbot &= \frac{(\mu^{(k)} - \mu(x))^2}{\mu^{(k)}},\quad (\mu(x) < \mu^{(k)}).%,\text{ for}~ \mu(x) < \mu^{(k)}
% \end{align*}
describing how close  $\mu(x)$ is to
the mean rate of the highest background emitter.

\begin{theorem}(Sample and Runtime Guarantees).
 \label{thm:snakeucb_runtime_topk_inmain}
  Define the general adaptive and uniform sample complexity terms $\Hadaptk$ and $\Hunifk$:
  %for any integer $k \geq 1$:
\begin{align}
\label{eq:sample_complexity_k}
\Hadaptk &:= |S|\tau_0 + \sum_{x \in \calS^*(k)} \tfrac{1}{\divergtop} + \sum_{x
              \in \calS \setminus \calS^*(k)} \tfrac{1}{\divergbot} \nonumber\\
\Hunifk &:= |\calS| \tau_0 + |\calS| \frac{1}{\diverg(\mu^{(k+1)},\mu^{(k)})}
\end{align}

$\Hadaptk \geq \Hunifk$ for any integer number of sources $k \geq 1$ and any  distribution of emitters. For any $\deltot \in (0,1)$, the following hold each with probability at least $1-\deltot$:~\footnote{$\BigOhTil(\cdot)$ notation suppresses doubly-logarithmic factors.}  \\
(i)
$\snakeucb$ correctly returns $\Sstk$,
  with runtime at most
   \begin{align*}
    \Truntime(\snakeucb) \leq~& \Hadaptk \cdot
    %\Tsample + \cardS \log_+ \left(\frac{\log_+\tfrac{\cardS}{\delta}}{\diverg(\mu^{(k+1)},\mu^{(k)})} \right) ~=~
    \BigOhTil\left( \log ({\cardS}/ {\deltot})\right) \\
    &\quad+\BigOhTil\left(\cardS \log_+\left( {\Hunifk}/ {\cardS}\right)\right)~.
    %&=& \calO\left(\Hadapt\log(M/\delta) + \tau_0 M^2 \log(\Hunif)\right)~.
  \end{align*}
  (ii) $\naivesnake$ correctly returns $\Sstk$ with runtime bounded by
  \begin{eqnarray*}
   \Truntime(\naivesnake) \leq  \Hunifk \cdot \BigOhTil\left(\log(\cardS/\deltot)\right)~.
  \end{eqnarray*}
  \end{theorem}

\iftoggle{ieee}{\begin{IEEEproof}}{\begin{proof}}
(Sketch) The runtimes~\eqref{eq:runtime_defs} of each algorithm depend on how quickly we can reduce the set $S_i$ in each round. For each point $x$, let $\ifin(x)$ denote the round at which $\snakeucb$ removes $x$ from $\calS_i$; at this point we are confident as to whether or not $x$ is in $\Sstk$, so we do not sample it on successive rounds. At round $i$, we spend time $\tau_i = 2^i$ sampling each point still in $\calS_i$, so that we spend $\sum_{x \in \calS} \sum_{i=0}^{\ifin(x)}\tau_i \leq \sum_{x \in \calS} 2^{1 + \ifin(x)}$ time sampling $x$ throughout the run of the algorithm. For $\naivesnake$, we sample all points in all rounds, so we spend time $\leq |\calS| \max_{x \in \calS} 2^{1 + \ifin(x)}$ sampling.

Now we bound $\ifin(x)$ for each algorithm. These quantities depend on the estimated means $\hat{\mu}_i(x)$. Using the concentration bounds that informed the bounding functions in Sec.~\ref{sec:ptwise_ci}, we can form deterministic bounds ${[\LCBbar_i, \UCBbar_i]}$ that depend only on the true means $\mu(x)$. We choose these to encompass the algorithm confidence intervals, so that: $\LCBbar_i(x) \leq \LCB_i(x) \leq \mu(x) \leq \UCB_i(x) \leq \UCBbar_i(x)$ with high probability. If each of these inequalities holds with probability $\deltot/(4|\calS|i^2)$, then a union bound gives that the probability of failure of any inequality over all rounds is at most $\deltot$. By Lemma~\ref{lem:main_correctness}, this ensures correctness with probability at least $ 1-\deltot$.

Because $\LCBbar_i(x)$ and $\UCBbar_i(x)$ are deterministic given $\mu(x)$ and are contracting to $\mu(x)$ nearly geometrically in $i$, we can bound $\ifin(x)$ by inverting the intervals to find the smallest integer $i$ such that
$\LCBbar_i(x^*) > \UCBbar_i(x)$ for all $x^* \in \Sstk$ and $x \in \calS \setminus \Sstk$.
This requires an inversion lemma from the best arm identification literature (Eq.~(110) in \cite{simchowitz2016best}). The specific forms of $\LCBbar_i$ and $\UCBbar_i$ yield the bounds on $\ifin(x)$ in terms of approximate KL divergences, which are added across all environment points to obtain the sample complexity terms for each algorithm in~\eqref{eq:sample_complexity_k}.

The form of $\Hunifk$ results from noting that the function $(a,b) \mapsto \frac{a}{(a - b)^2}$ is decreasing in $a$ and increasing in $b$ for $a > b$, and therefore
$
\max\left\{\max_{x \in \calS^*(k)} \frac{1}{\divergtop},    \max_{x \in \calS \setminus \calS^*(k)} \frac{1}{\divergbot}\right\} $ $= 1/\diverg(\mu^{(k+1)},\mu^{(k)}).
$
\iftoggle{ieee}{\end{IEEEproof}}{\end{proof}}

The $\BigOhTil\left(\log(|\calS|/\deltot)\right)$ term in the $\snakeucb$ runtime bounds accounts for travel times of transitioning between measurement configurations. The second term $|\calS| \log(\Hunif / |\calS|)$
%in the expression for $\Truntime(\snakeucb)$
accounts for the travel time of traversing the uninformative points in the global path $\trajpoints$ at a high speed.
This term is never larger than $\Truntime(\naivesnake)$ and is typically dominated by ${\Hadapt\cdot \BigOhTil\left(\log(|\calS|/\deltot)\right)}$.
With a uniform strategy, runtime scales with the \emph{largest value} of $1/d(\mu(x_1), \mu(x_2))$ over $x_1 \not\in
\Sstk, x_2 \in \Sstk$ because that quantity alone determines the number of rounds required. In
contrast, $\snakeucb$ scales with the \emph{average} of $1/d(\mu(x_1), \mu(x_2))$ because it dynamically chooses
which regions to sample more precisely.

In many scenarios, the number $k$ may
% rough
estimate the number of hotspots, or there may be more than $k$ sources with similarly high emissions. The extreme case is when emissions $\mu^{(k)}$ and $\mu^{(k+1)}$ are \emph{equal}, the divergences $\diverg(\mu^{(k+1)},\mu^{(k)})$ zero. Here the resultant sample complexities $\Hadapt$ and $\Hunif$ are infinite -- %The sample complexities correctly represent algorithm behavior:
because no statistical test can distinguish between these two emission $\mu^{(k)}$ and $\mu^{(k+1)}$, the algorithm continues to collect additional samples without terminating. Sec.~\ref{sec:multiple_source_close} resolves this issue by proposing a simple modification of the stopping condition which returns a set $\Shat$ of possibly greater than $k$ sources. This modification enjoys similar guarantees to our default termination rule (see Theorem~\ref{thm:snakeucb_runtime_topk_inmain_eps}).

\subsection{Sample complexity for heterogeneous sources}
Our sample complexity results qualitatively match standard bounds for
  active top-$k$ identification with sub-Gaussian rewards in the general
  multi-armed bandit setting (e.g. ~\cite{kalyanakrishnan2012pac}).
  The following corollary suggests that  when the
values of $\divergx$ are heterogeneous, $\snakeucb$ yields
significant speedups over $\naivesnake$.

\begin{corollary}(Performance under Heterogeneous Background Noise).
\label{prop:k_equal_one_results}
For a large environment with a single source $x^*$ with emission rate $\mu^*$ and background signals distributed as
$\mux \sim
\mathrm{Unif}[0,\mubar]$ for $x \ne x^*$, the ratio of the upper bounds on sample complexities of $\snakeucb$ to $\naivesnake$ scales with the ratio of $\mubar$ to $\mu^*$ as
$
1 - {\mubar}/{\mu^*}.
\label{eq:heterogenous signals}
$
\end{corollary}

\iftoggle{ieee}{\begin{IEEEproof}}{\begin{proof}}
% Since $k =
% 1$, we shall replace $\Sstk$ with $x^*$, with $\mu^* = \mu(x^*)$. We then only need to consider the divergences of non-source points with the source signal:
% {
% \begin{align*}
% \divergx := \frac{(\mu^* -  \mu(x))^2}{\mu^*},
% \end{align*}
% }
% The resulting sample complexity terms are then
% \iftoggle{ieee}
% {
%   \begin{align}
%  \Hadapt^{(1)} &:= |\calS| \tau_0 + \sum_{x \in  \calS \setminus \{x^*\}}\left(
%     \tfrac{1}{\divergx}\right) \nonumber \\
%     \Hunif^{(1)}   &:= |\calS| \tau_0 + |\calS| \max_{x \in \calS}
%   \tfrac{1}{\divergx}\label{eqn:complexity}~.
%   \end{align}
% }
% {
%   \begin{eqnarray}
%  \Hadapt := \sum_{x \in  \calS \setminus \{x^*\}}\left( \tau_0 +
%     \frac{1}{\divergx}\right) & \text{ and } &
%     \Hunif  := |\calS| \left( \tau_0 + \max_{x \in \calS}
%   \frac{1}{\divergx} \right)  \label{eqn:complexity}~.
% \end{eqnarray}
% }
To control the
complexity of $\naivesnake$, note
\begin{align*}
\Hunif = \BigOhTil(\max_{x
    \ne x^*} {1}/{\divergx}) = \BigOhTil(
  {\mu^*}/{(\mu^* - \max_{x \ne x^*}\mux)^2}).
  \end{align*}
It is well known that
that the maximum of $N$ uniform random variables on $[0,1]$ is approximately $1
- \Theta(\frac{1}{N})$ with probability $1 - \Theta(\frac{1}{N})$, implying
that $\max_{x \ne x^*}\mux \approx (1 - \frac{1}{|\calS|})\mubar \approx \mubar$ with probability at least $1 - \Theta({1}/{|\calS|})$. Hence, the sample complexity of $\naivesnake$ scales as
$
\BigOhTil\left({|\calS|\mu_*}/{(\mubar - \mu_*)^2}\right)
$.
On the other hand, the sample complexity of $\snakeucb$ grows as
\begin{align*}
\Hadapt =  \BigOhTil(\sum_{x \ne x^*}{1}/{\divergx}) = \BigOhTil(\sum_{x \ne x^*}\mu^*(\mu^* - \mux)^{-2}) .
\end{align*}
When $\mux \sim \mathrm{Unif}[0,\mubar]$ are random and $\cardS$ is large, the law of large numbers implies that this  tends to $\BigOhTil\left(\mu^* |\calS| \cdot \mathbb{E}_{ \mux \sim \mathrm{Unif}[0,\mubar]}{(\mu^* - \mux)^{-2}}\right)= \BigOhTil\left(|\calS| (\mu^* - \mubar)^{-1}\right)$.
%We can then compute
% \begin{align*}
% \mathbb{E}_{ \mux \sim \mathrm{Unif}[0,\mubar]}{(\mu^* - \mux)^{-2}}  = \frac{1}{(\mu^* - \mubar)\mu^*},
% \end{align*}
% Hence, the total complexity scales as
% \begin{align*}
% \BigOhTil\left(\mu^* |\calS| \cdot \mathbb{E}_{ \mux \sim \mathrm{Unif}[0,\mubar]}{(\mu^* - \mux)^{-2}}\right) = \BigOhTil\left(|\calS| (\mu^* - \mubar)^{-1}\right)
% \end{align*}
Therefore, the ratio of sample bounds of $\snakeucb$ to $\naivesnake$ is
$
\left(|\calS| (\mu^*-\mubar)^{-1}\right)/ \left({|\calS| \mu^*}{(\mu^*-\mubar)^{-2}} \right)= 1 - {\mubar}/{\mu^*}
$.
\iftoggle{ieee}{\end{IEEEproof}}{\end{proof}}

\subsection{Extension: unknown number of high-emission sources \label{sec:multiple_source_close}}
\newcommand{\must}{\mu^{*}}
In many scenarios, there may be considerably more sources of radiation than anticipated. For example, suppose that $\snakeucb$ is specified with one target source ($k=1$), but in fact there exist two sources $x_1,x_2 \in \calS$ with $\mu(x_1) = \mu(x_2) = \must$. Then, $\snakeucb$ will \emph{not} terminate with high probability, because it cannot differentiate between these two sources.

To remedy this, we can introduce a slightly more aggressive stopping criterion which will terminate even if multiple sources have similar emissions.
The stopping criterion can be stated to terminate if the following holds for an error parameter $\epsilon > 0$:
\begin{definition}[$\epsilon$-Approximate Termination Rule]\label{def:approx_terminate} Under the $\epsilon$-approximate termination rule, $\snakeucb$ either (a) terminates when $\calS_i = \emptyset$ and returns $\Shat \leftarrow \Stop_i$, or (b) terminates when
\begin{align}
\min \{\LCB_i(x) : x \in \calS_i\} \ge \max \{\UCB_i(x) : x \in \calS_i \} - \epsilon. \label{eq:approx_terminate}
\end{align}
and returns the \emph{union} of the sets $\Shat = \calS_i \cup \Stop_i$.
\end{definition}
This criterion will ensure bounded runtime even when there are multiple sources whose mean emissions are close to that of $\mu^{(k)}$. For another possible approach to addressing unknown $k$, we direct the reader to Sec.~\ref{ssec:unknown_sources}.
Under the modification of Definition~\ref{def:approx_terminate}, we can modify Theorem~\ref{thm:snakeucb_runtime_topk_inmain} as follows. For $\mu_2 \ge \mu_1 > 0$, introduce
\begin{align*}
\diverg_{\epsilon}(\mu_1,\mu_2) := \frac{\max\{\mu_2-\mu_1,\epsilon\}^2}{\mu_2}.
\end{align*}
Observe that $\diverg_{\epsilon}(\mu_1,\mu_2) \ge \diverg(\mu_1,\mu_2)$, and is always at least $\epsilon^2/\mu_2$.
We define complexity terms analogous to Eq.~\eqref{eq:sample_complexity_k}:
\begin{align}
\label{eq:sample_complexity_k_eps}
\Hadaptk(\epsilon) &:= |S|\tau_0 + \sum_{x \in \calS^*(k)} \tfrac{1}{\divergtopeps} \nonumber\\
&\qquad+ \sum_{x
              \in \calS \setminus \calS^*(k)} \tfrac{1}{\divergboteps} \nonumber\\
\Hunifk(\epsilon) &:= |\calS| \tau_0 + |\calS| \frac{1}{\diverg_{\epsilon}(\mu^{(k+1)},\mu^{(k)})}
\end{align}
The above describe the adaptive and uniform complexities analogous to those applied in Theorem~\ref{thm:snakeucb_runtime_topk_inmain}. Essentially, these complexities prevent the runtime from suffering if there are many sources whose emissions are close to that of $\mu^{(k)}$.

Lastly, we introduce a relaxed definition of correctness, which requires that an estimate set $\Shat$ of the top emitters contains all top emitters $\Sstk$, and all remaining sources in the set have emissions close to $\mu^{(k)}$.
\begin{definition}\label{defn:approx_correct} A set $\Shat \subset \calS$ is said to be $(k,\epsilon)$-correct if $\Shat \supseteq \Sstk$, and for any $x \in \Shat$, $\mu(x) \ge \mu^{(k)} - \epsilon$.
\end{definition}
For the stopping rule in Definition~\ref{def:approx_terminate} and approximate notion of correctness in Definition~\ref{defn:approx_correct}, Theorem~\ref{thm:snakeucb_runtime_topk_inmain} generalizes as follows:
\begin{theorem}(Sample and Runtime Guarantees).
 \label{thm:snakeucb_runtime_topk_inmain_eps}
For any $\deltot \in (0,1)$ and $\epsilon \le \mu^{(k)}$, the following hold each with probability at least $1-\deltot$: \\
(i)
$\snakeucb$ with the termination rule in Definition~\ref{def:approx_terminate} returns a $(k,\epsilon)$-correct $\Shat$ with runtime at most
   \begin{align*}
    \Truntime(\snakeucb) \leq~& \Hadaptk(\epsilon) \cdot
    %\Tsample + \cardS \log_+ \left(\frac{\log_+\tfrac{\cardS}{\delta}}{\diverg(\mu^{(k+1)},\mu^{(k)})} \right) ~=~
    \BigOhTil\left( \log ({\cardS}/ {\deltot})\right) \\
    &\quad+\BigOhTil\left(\cardS \log_+\left( {\Hunifk(\epsilon) }/ {\cardS}\right)\right)~.
    %&=& \calO\left(\Hadapt\log(M/\delta) + \tau_0 M^2 \log(\Hunif)\right)~.
  \end{align*}
  (ii) $\naivesnake$  with the termination rule in Definition~\ref{def:approx_terminate} returns a $(k,\epsilon)$-correct $\Shat$ with runtime bounded by
  \begin{eqnarray*}
   \Truntime(\naivesnake) \leq  \Hunifk(\epsilon)  \cdot \BigOhTil\left(\log(\cardS/\deltot)\right)~.
  \end{eqnarray*}
  \end{theorem}

  \begin{proof}[Proof Sketch]

  \textbf{Correctness} To see that the modified termination rule returns a $(k,\epsilon)$-correct set, we consider the two possible termination criteria. If the default stopping rule $\calS_i = \emptyset$, is triggered, then the correctness follows from Theorem~\ref{thm:snakeucb_runtime_topk_inmain}.

 Suppose instead that the stopping rule in Eq.~\eqref{eq:approx_terminate} is triggered, so that $\Shat = \calS_i \cup \Stop_i$. By the original correctness analysis, with high probability, the top emitters are never eliminated from $\calS_i \cup \Stop_i$. Thus, $\Shat$ contains $\Sstk$. To prove $(k,\epsilon)$-correctness, let us now show that for any $x \in \Shat$, $\mu(x) \ge \mu^{(k)} - \epsilon$. With high probability, $\Stop_i \subset \Sstk$ for all rounds $i$, so we will verify this while restricting to $x \in \calS_i$.  We note that on the high probability event that top emitters are never eliminated from $\calS_i \cup \Stop_i$, and if $|\Stop_i| < k$, there exists some $x_0 \in \Sstk$ such that $x_0 \in \calS_i$; for this source, $\mu(x_0) \ge \mu^{(k)}$. By definition of the stopping rule, we then have
  \begin{align}
  \forall x  \in \calS_i, \ \ \LCB_i(x)  \ge \UCB_i(x_0) - \epsilon. \label{eq:approx_terminate}
  \end{align}

  Under the high-probability event that the confidence intervals are correct, this means that
  \begin{align}\label{eq:correct_conseq}
  \forall  x \in \calS_i, \ \ \mu(x) \ge \max \{\mu(x): x \in \calS_i \} - \epsilon.
  \end{align}
  Hence, Eq.~\eqref{eq:correct_conseq} implies that
\begin{align*}
  \forall  x \in \calS_i, \mu(x) \ge \mu^{(k)} - \epsilon,
\end{align*} as desired.

  \textbf{Sample Complexity} Let us now account for the improved sample complexity. In the proof of Theorem~\ref{thm:snakeucb_runtime_topk_inmain}, we considered an upper bound $\ifin(x)$ on the last round $i$ at which $x \in \calS$ remains in $\calS_i$. The total number of samples where then $2^{\ifin(x)}$, and up to logarithmic factors, this upper bound scaled as $\frac{1}{\diverg(\mu^{(k+1)},\mu(x))}$ for $x \in \Sstk$, and with $\frac{1}{\diverg(\mu(x),\mu^{(k)})}$ for $x \notin \Sstk$.

  For the new stopping rule, we have two cases: if $\mu(x) \notin [\mu^{(k)} - \epsilon/4,\mu^{(k)} + \epsilon/4]$, then the desired sample complexity analysis carries through, as is. This is is because, if $\mu(x) \ge \mu^{(k)} + \epsilon/4$  then $\diverg_{\epsilon}(\mu^{(k+1)},\mu(x)) = \Omega(\diverg(\mu^{(k+1)},\mu(x))$ (and analogously when $\mu(x) \le \mu^{(k)} - \epsilon/4$).

  Now, consider $i_0$ to be sufficiently large that $i_0 \ge \max\{\inf(x) : \mu(x) \notin [\mu^{(k)} - \epsilon/4,\mu^{(k)} + \epsilon/4]\}$. Then, for $i \ge i_0$, all that remain are means $\mu(x) \in [\mu^{(k)} - \epsilon/4,\mu^{(k)} + \epsilon/4]$. But once $i$ is sufficiently small that the confidence intervals are at most $\epsilon/8$ in width, all LCB's and UCB's will be within, say $\epsilon/4$ of one another, triggering the new stopping condition Eq.~\eqref{eq:approx_terminate}. Hence, the sample complexity for these means is governed by how many samples are required to shrink these intervals to a width of $\Omega(\epsilon)$, which is roughly bounded by $\frac{\mu^{(k)} + \epsilon/4}{\epsilon^2}$. Under the assumption that $\mu^{(k)} \le \epsilon$, this is at most $\frac{\mu^{(k)}}{\epsilon^2}$, which is bounded by $1/\diverg_{\epsilon}(\mu^{(k)},\mu(x))$ for $x \notin \Sstk$, and by $1/\diverg_{\epsilon}(\mu(x),\mu^{(k+1)})$ for $x \in \Sstk$.
  \end{proof}

%% file: results.tex
% !TEX root = main_ieee.tex

\section{Experiments}
\label{sec:results}

\begin{figure*}[ht!]
  \centering
  \iftoggle{ieee}{
  \includegraphics[width=.9\textwidth]{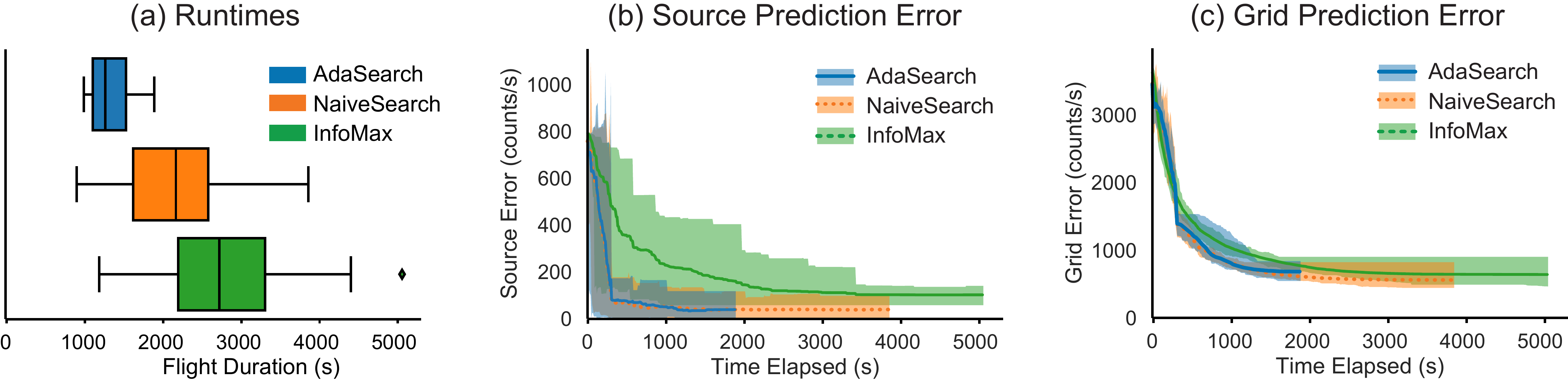}}
  {\includegraphics[width=1.0\textwidth]{figures_tro/fig1_panel_no_dots.pdf}}
  \caption{
       Simulation results of $\snakeucb$, $\naivesnake$ and $\infomax$ for 25 randomly instantiated environments with $\overline{\mu} = 400$ and $\mu^* =
    800\ \frac{\text{counts}}{\text{s}}$. Lighter shaded areas denote the range
    of values at each time over $25$ runs of each algorithm;
    dark lines show the mean. We include final errors for runs that have
    already finished in max, min, and mean computations. Each algorithm was given the same $25$
    randomized grids.
(a) Runtimes of $\snakeucb$, $\naivesnake$, and $\infomax$ over $25$ random trials.
(b) Absolute source error $|\hat{\mu}(x^*)(t)- \mu(x^*)|$ over time.
(c) Total grid error $\sqrt{\sum_{x \in \calS} \|\hat{\mu}(x)(t) - \mu(x)\|_2^2}$ over time. } %(measured as ratio of $\naivesnake$ to $\snakeucb$ runtime)
 \label{fig:comparison_results}
\end{figure*}

We compare the performance of $\snakeucb$
with the baselines defined in Sec.~\ref{sec:baseline} in simulation for the $\RSS$ problem with physical sensing model defined in~\ref{sec:physical_ci}, and
validate $\snakeucb$ in a hardware demonstration.

%\textbf{Simulation methodology.}
\subsection{Simulation methodology}
We evaluate $\snakeucb$, $\infomax$, and
$\naivesnake$ in simulation using the Robot Operating System (ROS) framework \cite{quigley2009ros}. Environment points $\calS$ lie in
a $16\times16$ planar grid, spread evenly over a  total area $64 \times 64$
$m^2$. Radioactive emissions are detected by a simulated sensor following the physical
sensing model given in Sec~\ref{sec:physical_ci} and constrained to fly above a minimum height of $2m$ at all times (see inset of Fig.~1 for simulation setup with $4 \times 4$ planar grid environment). For all experiments, we set confidence parameter $\alpha = 0.0001$.

For the first set of experiments (Figs.~\ref{fig:comparison_results},~\ref{fig:rel_speed}), we set $k = 1$, so that the set of sources $\calS^*(k) = \{x^*\}$ is a single point in the environment.
We set $\mu^* =
\mu(x^*) = 800$ photons/s.
In this setting, we investigate algorithm performance in the face of heterogeneous background signals by varying a maximum environment emission rate parameter  $\overline{\mu} \in
\{300,400,500,600\}$. For each setting of $\overline{\mu}$, we test all three algorithms on $25$ grids randomly generated with background emission rates drawn uniformly at random from the interval
$[0,\overline{\mu}]$.

We also examine the relative performance of all three algorithms as the number of sources increases  (Fig.~\ref{fig:k_greater_one}). For all experiments with $k>1$, we randomly assign $k$ unique environment points from the grid as the point sources, with emissions rates set to span evenly the range $[800,1000]$ photons/s. The signals of the remaining background emitters are drawn randomly as before, with $\overline{\mu}=400$.

Finally, we examine the relative performance of each algorithm for different environment sizes: square grids with  widths $2\times$ and $4 \times$ that of the previous experiments (\Cref{table:env_scale_factor}). To keep the number of sources that must be disambiguated consistent, we instantiate environment points $\calS$ in a $16\times 16$ grid, so that the size of each cell changes with the environment scale factor. Here we set $k=1$, $\mubar=400$, and $\mu^*=800$. For the $\infomax$ baseline, we adjust the planning horizon to scale with the width of the environment, setting $T_{\text{plan}} = 60$ for the doubled grid and $T_{\text{plan}} = 120$ for the largest grid.

\subsection{Results}
Figure~\ref{fig:comparison_results} shows performance across the three algorithms with respect to the following
metrics: %, which include measures of
% global mapping performance, source localization, and source intensity
% estimation:
(a) total runtime (time from takeoff until $x^*$ is located with confidence), (b) absolute difference between the predicted and actual
emission rate of $x^*$, and (c) aggregate difference between predicted and actual emission rates for all environment points $x \in \calS$, measured in Euclidean norm.
%Fig.~\ref{fig:comparison_results}(a) plots the
%runtimes for each algorithm for $\mubar = 400$ and $k=1$.
The uniform baseline
$\naivesnake$ terminates significantly earlier than $\infomax$, and $\snakeucb$
terminates even earlier, on average.
Of these $25$ runs, $\snakeucb$ finished faster than $\naivesnake$ in 21 runs, and finished faster than $\infomax$ in 24. The flight patterns for the first trial of each algorithm are shown in Fig.~\ref{fig:paths}.
%($\naivesnake$ finished faster than $\infomax$ in 18 of 25 runs).

To examine the variation in runtimes due to factors other than the environment instantiation, we also conducted $25$ runs of the same exact environment grid. Due to delays in timing and message passing in simulation (just like there would be in a physical system), measurements of the simulated emissions can still be thought of as random though the environment is fixed. Indeed, the variance in runtimes was comparable to the variance in runtimes in Fig.~\ref{fig:comparison_results}; over the $25$ trials of a fixed grid, the variance in runtimes were $265s$ ($\snakeucb$), $537s$ ($\naivesnake$), and $1028s$ ($\infomax$). Of these $25$ runs, $\snakeucb$ finished faster than $\naivesnake$ in 18 runs, and finished faster than $\infomax$ in all 25.
Fig.~\ref{fig:comparison_results}(b) plots the absolute
difference in the estimated emission rate $\hat \mu(x^*)$ and the true emission rate $\mu(x^*)$ at the one source.
$\snakeucb$ and $\naivesnake$ perform comparably over time, and
$\snakeucb$ terminates significantly earlier. %, likely due to having ruled out other high-emission background points quicker.
Fig.~\ref{fig:comparison_results}(c) plots the Euclidean error between the
estimated and the ground truth grids; in this metric the gaps in error between all three algorithms are smaller. $\snakeucb$ is fast at locating the highest-mean sources without sacrificing performance in total environment mapping.

Fig.~\ref{fig:rel_speed} shows performance of all three algorithms across different maximum background radiation thresholds $\mubar$. As $\mubar$ increases, all algorithms take longer to terminate because the source is harder to distinguish from increasing heterogeneous background signals (left). For high background radiation values (e.g. $\overline{\mu}=600$), the difference in runtimes between all three algorithms is larger; the runtime of $\snakeucb$ increases gradually under high background signals, whereas $\naivesnake$ and $\infomax$ are greatly affected. Fig.~\ref{fig:adap_vs_naive} shows that as $\overline{\mu}$ approaches $\mu^*$, the relative speedup of using adaptivity, $\Truntime(\naivesnake)/\Truntime(\snakeucb)$, increases. This is consistent with the
theoretical analysis in Sec.~\ref{sec:analysis}; the dashed line plots a fit curve with rule $0.7 \cdot \mu^*/(\mu^*
- \bar{\mu})$.

\begin{figure}[!h]
  \centering
  \iftoggle{ieee}{
  \includegraphics[width=.75\columnwidth]{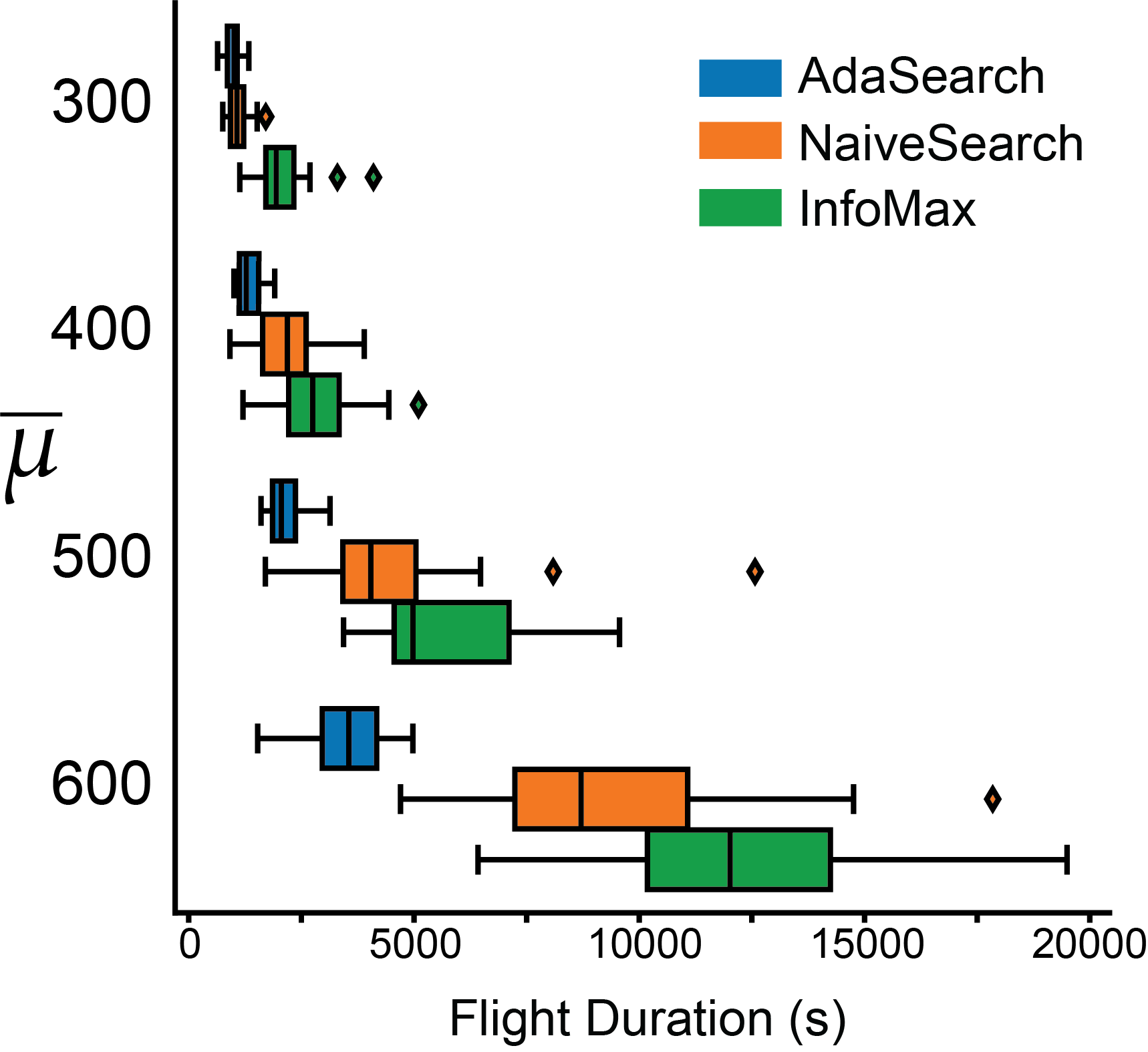}}
  {\includegraphics[width=.5\columnwidth]{figures_tro/background_barplot.png}}
  \caption{
  Performance %(measured as ratio of $\naivesnake$ to $\snakeucb$ runtime)
  of all three algorithms for grids with maximum background varying, and set as
    $\overline{\mu} \in \{300, 400, 500, 600\}$, $\mu^* = 800\
    \text{counts}/\text{s}$. For each value of $\overline{\mu}$, each algorithm was given the same 25
    randomized grids.} %Box plots show quartile values.}
  \label{fig:rel_speed}
\end{figure}
\begin{figure}[!h]
  \centering
  \iftoggle{ieee}{
  \includegraphics[width=.65\columnwidth]{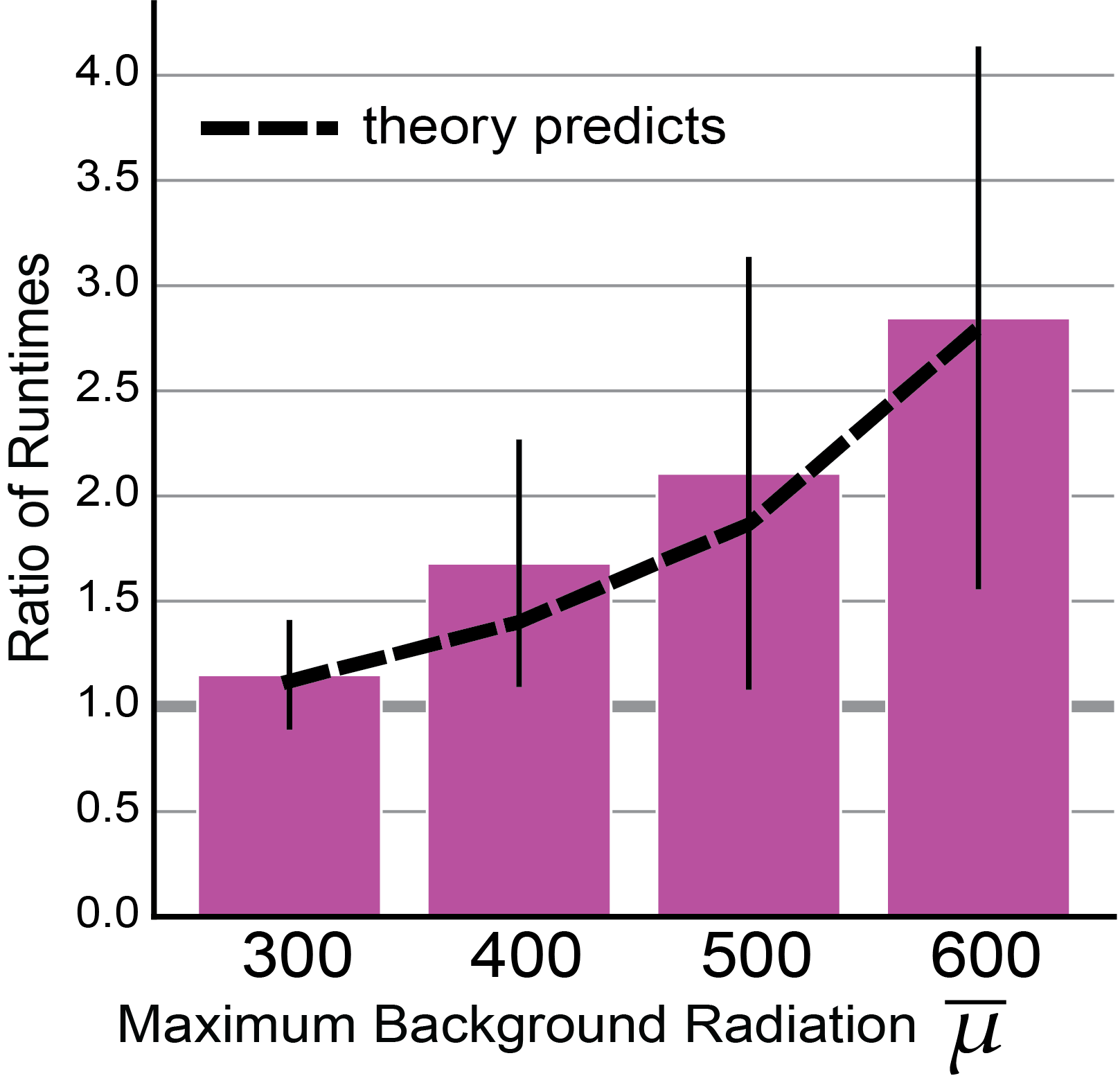}}
  {\includegraphics[width=.5\columnwidth]{figures_tro/background_boxplot.png}}
  \caption{
  Relative performance %(measured as ratio of $\naivesnake$ to $\snakeucb$ runtime)
  of $\snakeucb$ and $\naivesnake$ per random grid, measured as $\Truntime(\naivesnake)/\Truntime(\snakeucb)$; magenta bars indicate mean, and horizontal black lines denote one standard deviation from the mean in either direction. Dashed line shows an approximate fit according to \cref{prop:k_equal_one_results}. Data is obtained from the same randomized grids as in Fig.~\ref{fig:rel_speed}.} %Box plots show quartile values.}
  \label{fig:adap_vs_naive}
\end{figure}

Fig.~\ref{fig:k_greater_one} compares algorithm runtimes across different
numbers of sources, $k$. As suggested from
Corollary~\ref{thm:snakeucb_runtime_topk_inmain}, both absolute and relative
performance is
consistent across $k$ for all three algorithms.

\begin{figure}[!h]
  \centering
  \iftoggle{ieee}{
  \includegraphics[width=.8\columnwidth]{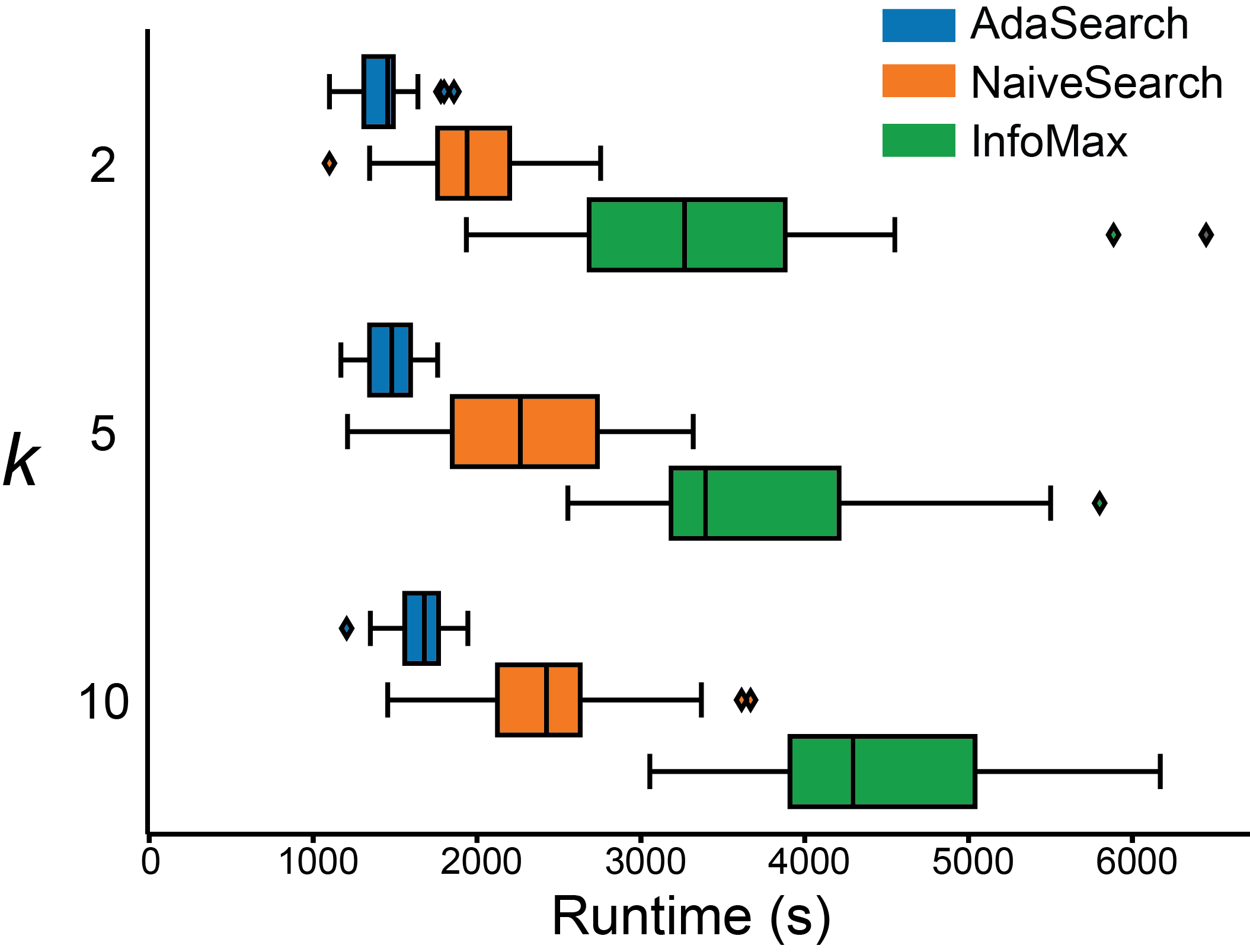}}
  {\includegraphics[width=.5\columnwidth]{figures_tro/k_greater_one.png}}
  \caption{Performance across number of sources, $k$ for $k \in [2,5,10]$. $\mu(x) \in [800,1000]$ for all $x \in \calS^*(k)$. For each value of $k$, Each algorithm was given the same 25
    randomized grids.} %Box plots show quartile values.}
  \label{fig:k_greater_one}
\end{figure}

  The runtimes and number of rounds executed by $\snakeucb$ and $\naivesnake$ for these experiments are summarized in \Cref{table:number_of_rounds}. In every trial, $\snakeucb$ takes no more rounds than $\naivesnake$ to reach the termination criterion, suggesting that slowing down over informative points saves the algorithm from having to do more entire passes over the environment. Note that each successive round of the $\snakeucb$ algorithm takes longer than a round of $\naivesnake$, since the robot slows down over informative regions, whereas $\naivesnake$ does not.

 $\snakeucb$ is inherently a probabilistic algorithm, returning the true sources with probability $1-\deltot$, as a function of the number of rounds and the confidence with parameter, $\alpha$. Of the 175 trials run throughout these experiments, $\snakeucb$ locates the correct source in 174 of them ($99.4 \%$). We set $\alpha = 0.0001$ in our experiments to facilitate fair comparison of algorithms while maintaining reasonable runtime of the slower methods ($\naivesnake$, $\infomax$). Given the speed with which $\snakeucb$ returns a source, in practice it would be feasible to reduce $\alpha$, and hence reduce the probability of a mistake, $\deltot$.
Due to the good performance of total grid mapping (Fig.~\ref{fig:comparison_results}(c)), even in the low-probability case that an incorrect source is returned, $\snakeucb$ still provides valuable information about the environment.

\begin{table}[ht]
\centering
{
% To place a caption above a table
\caption{  Round number at time at termination and runtime for $\snakeucb$ ($\mathtt{Ada}$) and $\naivesnake$ ($\mathtt{Naive}$). Averages and standard deviations taken over 25 trials.}
\label{table:number_of_rounds}
\begin{tabular}[t]{cccccc}
\toprule
% <<<<<<< HEAD
& & \multicolumn{2}{c}{round \#  at term.:  avg (std) }  & \multicolumn{2}{c}{runtime in seconds:  avg (std) }  \\
\cmidrule(lr){3-4}  \cmidrule(lr){5-6}
% =======
% & & \multicolumn{2}{c}{\# rounds:  avg (std) }  & \multicolumn{2}{c}{runtime in seconds:  avg (std) }  \\
% \cmidrule(lr){3-4}  \cmidrule(lr){5-6}
% >>>>>>> 1b3ed8084453997f02e234baf44932a26c01ac12
k &  $\overline{\mu}$ & $\mathtt{Ada}$  & $\mathtt{Naive}$  & $\mathtt{Ada}$  & $\mathtt{Naive}$  \\
\midrule
1 & 300 &  3.0 (0.4) & 3.8 (0.8) & 981 (1778)  & 1103 (237)  \\
1 & 400  & 3.8 (0.6) & 7.4 (2.4) & 1352 (255) & 2208 (694)  \\
1 & 500 & 5.0 (0.7) & 14.4 (6.8) & 2136 (391) & 4436 (2145) \\
1 & 600 & 6.6 (0.6) & 30.5 (9.5) & 3558 (767) & 9483 (3004) \\
2 & 400 & 4.0 (0.4) & 6.6 (1.2) & 1442 (200) & 1955 (386) \\
5 & 400 & 3.8 (0.4) & 7.8 (2.0) & 1465 (167) & 2295 (569)\\
10 & 400 & 4.2 (0.5) & 8.2 (1.7) & 1667 (168) & 2502 (532)\\
\bottomrule
\end{tabular}
}
\end{table}%

The performance of each algorithm for environments at larger scale factors is given in~\Cref{table:env_scale_factor}. Doubling the environment scale factor has two effects on the difficulty of the problem. First, it essentially doubles the flight time to fulfill a "snaking" path across the environment (see Fig.~\ref{fig:snake_pattern}).  Second, doubling the environment grid width distributes environment points further from each other in space, so that contributions from individual environment points are easier to disambiguate. The results in \Cref{table:env_scale_factor} show that for larger grid sizes, both $\snakeucb$ and $\naivesnake$ outperform $\infomax$ in terms of runtime. Additionally, the difference in average runtime between $\snakeucb$ and $\naivesnake$ is small for the larger grid sizes ($128 \times 128$ and $256 \times 256~m^2$), a consequence of the easier sampling problem (due to dispersed environment points), indicated by a reduced number of rounds needed for the larger grid environments, compared to the $64 \times 64~m^2$ grid.
  In all runs summarized in \Cref{table:env_scale_factor}, the algorithms locate the correct source.

\begin{table*}[!t]
\centering

% To place a caption above a table
\caption{Runtime for each algorithm with different environment sizes, as well as the round number at termination for $\snakeucb$ and $\naivesnake$. Each grid consists of $16 \times 16$ cells with $k=1$ and $\overline{\mu} = 400$. Averages and standard deviations taken over 10 trials.}
\label{table:env_scale_factor}
\begin{tabular}[t]{lccccc}
\toprule
& \multicolumn{2}{c}{round \# at termination:  avg (std)}
& \multicolumn{3}{c}{runtime in seconds:  avg (std)}   \\
\cmidrule(lr){2-3} \cmidrule(lr){4-6}
\multicolumn{1}{c}{grid size} &
 $\snakeucb$  & $\naivesnake$&
 $\snakeucb$  & $\naivesnake$ & $\infomax$
 \\
\cmidrule(lr){1-1} \cmidrule(lr){2-2} \cmidrule(lr){3-3} \cmidrule(lr){4-4} \cmidrule(lr){5-5} \cmidrule(lr){6-6}
$64 \times 64~m^2 $& 3.8 (0.4) & 7.6 (2.4) & 1345 (229) & 2222 (692) & 2460 (701) \\
$128 \times 128~m^2 $& 2.0 (0.0) & 2.4 (0.5) & 1356 (170) & 1322 (193) & 3301 (837) \\
$256 \times 256~m^2 $& 1.9 (0.3) & 1.9 (0.3) & 1916 (548) & 2025 (424) & 6792 (1949) \\
\bottomrule
\end{tabular}

\end{table*}%

\begin{figure*}[!ht]
  \centering
  \iftoggle{ieee}{\includegraphics[width=.75\textwidth]{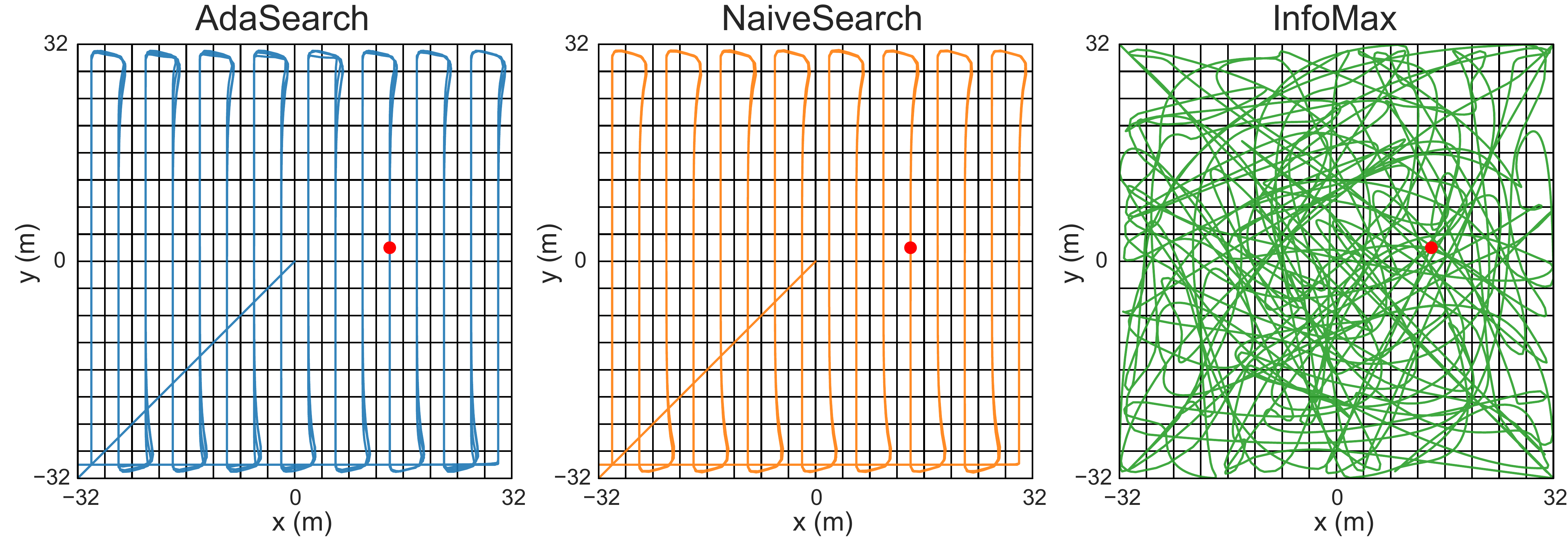}}
  {\includegraphics[width=1.0\textwidth]{figures_tro/paths.pdf}}
  \caption{Indicative flight paths for each algorithm.  The highest emitting source is denoted by the red dot.  $\snakeucb$ and $\naivesnake$ follow a fixed pattern over several rounds, whereas $\infomax$ does not.} %(measured as ratio of $\naivesnake$ to $\snakeucb$ runtime)
 \label{fig:paths}
\end{figure*}

\subsection{Discussion}
While all three methods eventually locate the correct source $x^*$ the vast majority of the time, the two
algorithms with global planning heuristics, $\snakeucb$ and $\naivesnake$,
terminate considerably earlier than $\infomax$, which uses a greedy, receding
horizon approach (Fig.~\ref{fig:comparison_results}). Moreover, the adaptive
algorithm $\snakeucb$ consistently terminates before its non-adaptive
counterpart, $\naivesnake$. These trends hold over differing background noise threshold $\mubar$ and number of sources, $k$ (Figs.~\ref{fig:adap_vs_naive} and~\ref{fig:k_greater_one}).

The $\snakeucb$ algorithm excels when it can quickly rule out points in early
rounds. From \eqref{eq:sample_complexity_k} we recall that the $\snakeucb$ sample complexity scales
with the average value of $\mu(x) / (\mu^* - \mu(x))^2$ (rather than the maximum,
for $\naivesnake$). Hence, $\snakeucb$ will outperform $\naivesnake$ when there
are varying levels of background radiation.

As $\mubar$ approaches $\mu^*$ and the gaps $\mu^* - \mu(x)$ become more
variable, adaptivity confers even greater advantages over uniform sampling. From
\cref{prop:k_equal_one_results}, we expect the ratio of $\naivesnake$  runtime to  $\snakeucb$ runtime to scale as ${\mu^*}/ ({\mu^* -
  \mubar})$, which is corroborated by the fit of the dashed line to the average runtime ratios in
Fig.~\ref{fig:k_greater_one}.
The stability of $\snakeucb$ in spite of increasing background noise is striking, especially in comparison to the two alternatives presented here; this suggests that in settings where background noise could be misleading to discerning the true signal, a confidence-bound based sampling scheme is likely preferable.

The performance differences between $\snakeucb$ and $\infomax$, and $\naivesnake$ hold as the number of sources increases, indicating that $\snakeucb$ is preferable for a range of different enviroments and \problemname\ instances.

$\infomax$'s strength lies in quickly reducing global uncertainty across the
entire emissions landscape. However, $\infomax$ takes considerably longer to
identify $x^*$ (Fig.~\ref{fig:comparison_results}(a)) and, surprisingly, $\snakeucb$
and $\naivesnake$ perform similarly to $\infomax$ in mapping the entire
emissions landscape on longer time scales (Fig.~\ref{fig:comparison_results}(c)). We
attribute this to the effects of greedy, receding horizon planning.
Initially, $\infomax$ has many locally-promising points to explore and reduces
the Euclidean error quickly. Later on, it becomes harder to find
informative trajectories that route the quadrotor near the few under-explored regions. The results in~\Cref{table:env_scale_factor} suggest that this problem remains for larger environments as well.
These results suggest that when a path $\config$ such as the
raster path used here is available, it is well worth considering.

High variation in all experiments is expected due to the noisyPoisson emissions signals. While this noise effects the runtime of all
algorithms, the range of runtimes for $\snakeucb$ is consistently tight compared to the other two methods, suggesting that
carefully allocated measurements are indeed increasing robustness under heterogeneous background signals.

%\textbf{Hardware results.}
\subsection{Hardware demonstration}

The previous results are based on a simulation of two
key physical processes: radiation sensing and vehicle dynamics.
We also test $\snakeucb$ on a Crazyflie 2.0
palm-sized quadrotor in a motion capture room with simulated radiation readings.
The motion
capture data (position and orientation) is acquired at roughly $235$ Hz and
processed in real time using precisely the same implementation of $\snakeucb$ used in our software simulations. Our supplementary video shows a more
detailed display of our system.\footnote{Video available at \url{https://people.eecs.berkeley.edu/~erolf/adasearch.m4v}.} %(TODO: reference video in the form the submission wants)}}
Fig.~\ref{fig:drone} visualizes
the confidence intervals and the absolute source point estimation error, as well
as the horizontal speed, during
a representative flight over a small $4 \times 4$ grid, roughly $3m$ on a side.

\begin{figure}[!ht]
  \centering
  \includegraphics[width=\columnwidth]{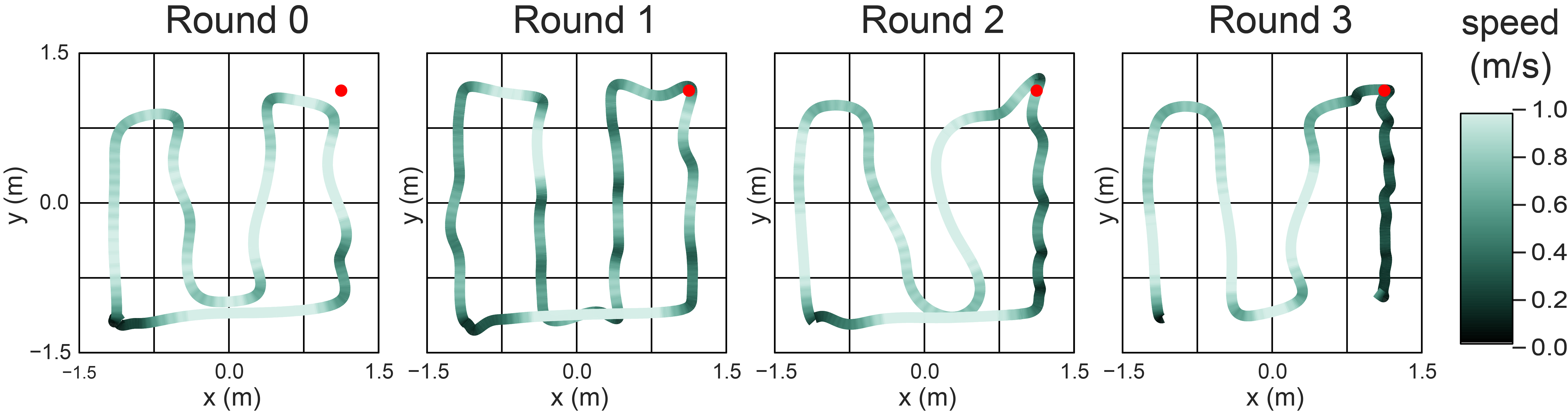}
  \caption{Hardware experiment trajectories for each round of $\snakeucb$, color coded by speed ($m/s$). The highest emitting source is denoted by the red dot. } %Box plots show quartile values.}
  \label{fig:hardware_speed_positions}
\end{figure}

Fig.~\ref{fig:hardware_speed_positions} shows the flight paths for each round, color coded by speed (darker is slower). Despite imperfections in following the snake path and velocity changes, the robot's trajectory successfully represents the algorithm.
After two rounds, $\snakeucb$ identifies the two highest emitting points as
the highlighted pixels in the top inset, and the
absolute error in estimating $\mu(x^*)$ is very small. $\snakeucb$ spends most of
its remaining runtime sensing these two points and avoids taking
redundant measurements elsewhere. The plot of horizontal speed over time
(lower inset of Fig.~\ref{fig:drone}){}
shows this reallocation of sensor
measurements; in the final two rounds, the quadrotor moves quickly at first, then slows down over the two candidate points.
This hardware demonstration gives preliminary validation that $\snakeucb$ is indeed safe and reasonable to use onboard a physical system.

%% file: generalizations.tex
%!TEX root = main_ieee.tex

\section{Generalizations and Extensions}
\label{sec:generalizations}
Before concluding, we briefly discuss several extensions and generalizations of $\snakeucb$.
%\textbf{Unknown number of sources.}
\subsection{Unknown number of sources \label{ssec:unknown_sources}}
In~\ref{sec:multiple_source_close}, we presented a modified termination criterion which accomodates the possibility of multiple high-emission sources (Defn.~\ref {def:approx_terminate}). 
This criterion is particularly suitable if there are multiple sources whose emissions are near that of the $k$-th largest source $\mu^{(k)}$. 
The modified algorithm correctly recovers all top-$k$ sources, as well as possibly returning \emph{some} additional sources for which emissions $\mu(x)$ are within $\epsilon$ of $\mu^{(k)}$.

In other scenarios, it may be more suitable to return \emph{all} sources for which emissions $\mu(x)$ are within $\epsilon$ of $\mu^{(k)}$ and \emph{no} sources whose emission are with a factor of $\epsilon' > \epsilon$. With more sophisticated termination criterion and candidate sets $\Stop_i$ and $\calS_i$, $\snakeucb$ can be modified to accomodate this alternate guarantee. More broadly, the $\snakeucb$ design principle - combining  confidence-interval based elimination with simple raster movement planning - is ammenable to other approximate-search criteria which may arise in given application domains. 

We also note that, if one runs $\snakeucb$ with a small $k$, the algorithm will still collect  measurements from other high-emission locations that can re-used if the practicioner wishes to consider a greater number of $k' > k$ on a subsequent run.

%If the number of sources is initially
%unknown, then running $\snakeucb$ with small $k$ will result in measurements
%that are still informative about all true sources, since the additional unknown sources must be
%distinguished from the top $k$ sources. This could result in sufficient measurement
%coverage, or it could as a first pass which would inform the choice of a larger $k$ in a subsequent run of the algorithm.

%\textbf{Oriented sensor.}
\subsection{Oriented sensor}
A natural extension of the radioactive source-seeking example is to consider a sensing model
with a sensitivity function which depends upon orientation. The additional challenge
lies in identifying informative sensing configuration sets $\config_i$ and a
reasonably efficient equivalent fixed global path $\trajpoints$.
More broadly, the sensing configurations $z \in \config$
could be taken to represent generalized configurations of the robot and sensor,
e.g., they could
encode the position and angular orientation of a directional sensor or joint
angles of a manipulator arm.

%\textbf{Pointwise sensing model.}
\subsection{Pointwise sensing model}
We motivated  the pointwise sensing model where sensitivity function is $h(x,z) = \I\{x = z\}$ as a model conducive to theoretical analysis.
Though it is only a coarse (yet still predictive) approximation of the physical process of radiation
sensing, this sensitivity model is a more precise descriptor of other
sensing processes. For example, the pointwise model is appropriate for
\iftoggle{ieee}{finding the most crowded waiting rooms in a hospital on average
  during a day, and for surveying remote populations to locate the highest incidence
  rate of a disease.}{survey design. As a concrete example, suppose an aid group with enough funding to set
up $k$ medical clinics sought to identify which $k$ towns had the highest rates of
disease.
It is reasonable to think that the data collected about town $i$ is mostly informative about only the rate of disease in that town, so that the pointwise sensing model may be quite appropriate.}

\iftoggle{ieee}{}{
  % \textbf{Surveying.}
  \subsection{Surveying}
Although we demonstrate $\snakeucb$ operating onboard a UAV in the context of
$\RSS$, the core algorithm applies more broadly, even to non-robotic embodied sensing problems.
% Consider the problem of in-person surveying of remote areas
% to find the highest populations of a rare species.
Consider the problem of planning $k$ clinic locations. Because surveys are conducted in person, the aid group is resource limited in terms of using human surveyors, both in terms of the time it takes to survey a single person or clinic within a town, and in terms of travel time between towns.
A survey planner could use $\snakeucb$ to guide the decisions of how long to spend in
each town counting new cases of the disease before moving on to the next, and to trade-off the travel time of returning to collect more data from a certain town with spending extra time at the town in the first place.

While $\snakeucb$ provides a
good starting point for solving such problems, the high cost of transportation
would likely make it worthwhile to further optimize the surveying trajectory at
each round, e.g. by (approximately) solving a traveling salesman problem.
}

%% file: conclusions.tex
%!TEX root = main_ieee.tex

\section{Conclusion}
\label{sec:conclusion}
\iftoggle{ieee}{We}{In summary, we}
 have shown that %with a careful formulation of the problem,
statistical
methods from pure exploration active learning offer a promising, under-explored toolkit for
robotic \problemname. %We have characterized and evaluated a
%theoretical framework for planning in a variety of dynamic sensing problems,
% and have shown success in a case study on radiation-detection.
Specifically, we have shown that motion
constraints need not impede %existing
active learning strategies.%, and highlighted incorporating realistic
                           %measurement models as fertile ground for future research.

Our main contribution, $\snakeucb$, outperforms a greedy
information-maximization baseline in a radioactive source-seeking task. Its success can be understood as a consequence of two
structural phenomena: planning horizon and implicit design objective. The
information-maximization baseline operates on a receding horizon and seeks to
reduce global uncertainty, which means
that even if its planned trajectories are individually highly informative, they
may lead to suboptimal performance over a long time scale.
In contrast,
$\snakeucb$ uses an application-dependent global path that efficiently covers the entire search space and allocates
measurements using principled, statistical confidence intervals.

While our results for the problem of RSS are encouraging, it is likely that in may applications, performance could be limited by the range, field of view, or orientation of the sensors. In some cases (e.g. oriented sensors), such limitations could be addressed by the extensions suggested in Sec.~\ref{sec:generalizations}, and in others, might necessitate new innovations. We are hopeful that the abstraction of sensing models, statistical measurement, and path planning as separate but integrated components of source seeking can guide such future innovations.

% Currently, $\snakeucb$ requires \textit{a priori} knowledge of the possible
% locations of a discrete set of points of interest. In problems where these
% locations are unknown or belong to a continuum, information maximization methods
% may be more suitable.
%Future work will address such problems.
$\snakeucb$  excels in situations with a heterogeneous
distribution of the signal of interest; it would be interesting to make a direct
comparison with Gaussian process (GP)-based methods in a domain where the smooth GP priors are more appropriate.
 We also plan to explore active sensing in more complex environments and with
dynamic signal sources and more sophisticated sensors (e.g., directional
sensors). Furthermore, as $\snakeucb$ is explicitly designed for general
embodied sensing problems, it would be exciting to test it in a wider variety of application domains.

%% file: lemma_one_proof.tex
%!TEX root = main_arxiv.tex
\section{Proof of Lemma~\ref{lem:main_correctness} }
\label{app:main_correctness}
We verify Lemma~\ref{lem:main_correctness} given in
Section~\ref{sec:algorithm_snake}. The proof of this lemma holds \emph{for any} instantiation of Algorithm~\ref{snake_lucb++}, regardless of the sensing model or the planning strategy.

First, we verify that for each round $i \ge 0$, $\Stop_{i} \cap \calS_i = \emptyset$. Indeed, at round $i = 0$, $\Stop_i = \emptyset$, so the bound holds immediately. Suppose by an inductive hypothesis that $\Stop_{i} \cap \calS_i = \emptyset$ for some $i \ge 0$. Then, for any $x \in \Stop_{i+1}$, we have two cases:
\begin{itemize}
	\item[(a)] $x \in \Stop_i$. Then, $x \notin \calS_{i}$ by the inductive hypothesis, and $\calS_i \supset \calS_{i+1}$ by~\eqref{eq:bot_elim}.
	\item[(b)]  $x$ is added to $\Stop_{i+1}$ via~\eqref{eq:top_elim}. Then $x \notin \calS_{i+1}$ by~\eqref{eq:bot_elim}.
\end{itemize}
Next, we verify that if the confidence intervals are correct in all rounds leading up to round $i$, i.e.
\begin{align}\label{eq:correct_coverage2}
\text{for all rounds }j \le i \text{ and all }x\in\calS_j, \quad \mu(x) \in (\LCB_{j}(x),\UCB_{j}(x)),
\end{align}
then $\Stop_{i+1} \subset \Sstk$, and $\Sstk \subset \Stop_{i+1} \cup \calS_{i+1}$. We again use induction. Initially, we have $\Stop_0 = \emptyset \subset \Sstk \subset \calS = \calS_0$. Now, suppose that at round $i$, one has that  $\Stop_{i} \subset \Sstk$, and $\Sstk \subset \Stop_{i} \cup \calS_{i}$.

To show that $\Stop_{i+1} \subset \Sstk$, it suffices to show that if $x$ is added to $\Stop_{i+1}$, then $x \in \Sstk$. By the inductive hypothesis there exists $k - |\Stop_i|$ elements of $\Sstk$ in $\calS_i$. Hence, if $x$ is added to $\Stop_{i+1}$, and if~\eqref{eq:correct_coverage2} holds, then
\begin{align*}
\mu(x) &\ge \LCB_i(x) \\
&> (k-|\Stop_i|+1)\text{-th largest value of } \UCB_i(x),~x \in \calS_i \tag*{by~\eqref{eq:top_elim}}\\
&\ge (k-|\Stop_i|+1)\text{-th largest value of } \mu(x),~x \in \calS_i \tag*{by~\eqref{eq:correct_coverage2}}\\
&\ge (k+1)\text{-th largest value of } \mu(x), x \in \calS_i \cup \Stop_i. 
% \\
% &\ge (k+1)\text{-th largest value of } \mu(x),~x \in \calS_i \cup \Stop_i.
\end{align*}
Hence, $\mu(x)$ is among the $k$ largest values of $\mu(x)$ for $x \in \calS_i \cup \Stop_i$. Since $\Sstk \subset \calS_i \cup \Stop_i$, we therefore have that $x \in \Sstk$.

Similarly, to show $ \Sstk \subset \calS_{i+1} \cup \Stop_{i+1}$, it suffices to show that if $x \in \calS_{i} \setminus \calS_{i+1}$, and $x \notin \Stop_{i+1}$, then $x \notin \Sstk$. For $x$ such that $x \in \calS_{i} \setminus \calS_{i+1}$, and $x \notin \Stop_{i+1}$, it follows that
\begin{align*}
\mu(x) &\le \UCB_i(x) \\
&< (k-|\Stop_{i+1}|)\text{-th largest value of } \LCB_i(x),~x \in \calS_i \tag*{by~\eqref{eq:bot_elim}}\\
&\le (k-|\Stop_{i+1}|)\text{-th largest value of } \mu(x),~x \in \calS_i \tag*{by~\eqref{eq:correct_coverage2}}\\
&\le k\text{-th largest value of } \mu(x), x \in \calS_i \cup \Stop_{i+1}\\
&\le k\text{-th largest value of } \mu(x), x \in \calS,
\end{align*}
hence $\mu(x) \notin \Sstk$.

Finally, we verify that if~\eqref{eq:correct_coverage2} holds at each round, then at the termination round $\ifin$, $\calS_{\ifin} = \emptyset$, so that $\Stop_{\ifin} \subset \Sstk \subset \Stop_{\ifin} \cup \calS_{\ifin} = \Stop_{\ifin}$, so that $\Sstk = \Stop_{\ifin}$.

%% file: analysis_intro.tex
%!TEX root = main_ieee.tex
\section{Theoretical Results for Pointwise Sensing\label{sec:main_theory_results}}
In this appendix, we present formal statements of the measurement complexities
provided in Sec.~\ref{sec:analysis} in the main text, and generalize them
to the full top-$k$ problem presented in Algorithm~\ref{snake_lucb++} of the
main text. We also provide specialized bounds for the
randomly generated grids considered in our simulations. 

\textbf{Notation:} Throughout, we shall use the notation $f \lesssim g$ to denote that there exists a universal constant $C$, independent of problem parameters, for which $f \le C \cdot g$.  We also define $\log_{+}(x) := \max\{1,\log x\}$.

\textbf{Formal Setup.} Throughout, we consider a rectangular grid
$\calS$ of $|\calS|$ points, and let $\mu(x)$
denote the mean emission rate of each point $x \in \calS$ in
counts/second. We let $\mu^{(k)}$ denote the $k$-th largest mean $\mu(x)$. In the case that $k = 1$, we denote $ \mu^* := \mu^{(1)}$, and let $x^* := \arg\max_{x \in \calS} \mux$ denote the
highest-mean point, with emission rate $\mu^* := \mu(x^*)$. For identifiability, we assume $\mu^{(k)} > \mu^{(k+1)}$.

\textbf{Measurements.} As described in the main text, we assume a point-wise
sensing model in which $\snakeucb$
and $\naivesnake$ can measure each point directly. Recall that , at each round $i$, $\snakeucb$ takes $\tau_i := 2^{i} \tau_0$ measurements at each point $x \in \calS_i$, and $\naivesnake$ takes $\tau_i$ measurements at each $x \in \calS$. We let $\cnt_i(x)$ denote the total number of counts collected at position $x$ at round $i$.
We further assume that $\mu(x)$ are
standardized according to the same time units as $\tau_0$, measuring a source of mean $\mu(x)$
for time interval of length $\tau_i$ yields counts
distributed according to $\cnt_i \sim \Poi(2^i \cdot \mu(x))$.  Finally, we shall let $\ifin$ denote the (random) round at which a given algorithm - either $\snakeucb$ or $\naivesnake$ - terminates. 

\textbf{Confidence Intervals} %\maxnote{Flesh out the description of the confidence intervals} 
At the core of our analysis are rigorous $1-\delta$ upper and lower confidence
intervals for Poisson random variables, proved in
Sec.~\ref{sec:poisson_conc1}:
\begin{proposition}\label{prop:poisson_conc1}
  Fix any $\mu \ge 0$ and let $\cnt \sim \Poi(\mu)$. Define
  \begin{eqnarray}
    U_+\left(\cnt,\delta\right):= 2 \log(1/\delta) + \cnt + \sqrt{2\cnt \log\left(1/\delta\right)} ~~\mathrm{and}~~
    U_-\left(\cnt,\delta\right) := \max\left\{0,\cnt  - \sqrt{2\cnt\log(1/\delta)}\right\}
  \end{eqnarray}
  Then, it holds that $\Pr[\mu > U_+(\cnt,\delta)] \le \delta$ and $\Pr[\mu <
  U_-(\cnt,\delta)] \le \delta $.
\end{proposition}
At each round $i$ and $x \in \calS_i$ ($\snakeucb$) or $x \in \calS$ ($\naivesnake$), recall that we use upper and lower confidence intervals
\begin{eqnarray*}
  \LCB_i(x) := \frac{1}{\tau_i}U_-\left(\cnt_i(x), \delta_i\right)~,\quad~
  \UCB_i(x) :=  \frac{1}{\tau_i} U_+\left(\cnt_i(x),\delta_i \right)~,\quad\text{where } \delta_i := \delta/{(4|\calS| i^2)}. 
  %\label{eqn:lcb_ucb}
\end{eqnarray*}

\textbf{Trajectory for $\snakeucb$.} $\snakeucb$ follows a
trajectory where, at each round $i$, $\snakeucb$ spends time $\tau_i
= 2^i \cdot \tau_0$ measuring each $x \in \calS_i$, and spends $\tau_0$ travel
time traveling over each $x \notin \calS_i$. For the radioactive sensing problem, this is achieved by following the ``snaking
pattern'' depicted in Fig.~\ref{fig:snake_pattern} in the main text, in which the quadrotor speeds up or slows down over each point to match the specified measurement times. 
 We will define the \emph{total sample complexity} and \emph{total run time} respectively as
\begin{align*}
\Tsample := \sum_{i = 0}^{\ifin}\tau_i |\calS_i|  ~\text{ and }~  \Truntime := \Tsample +  \sum_{i=0}^{\ifin}\tau_0|\calS \setminus \calS_i|
\end{align*}
 The first quantity above captures the total number of measurements taken at points we still wish to measure, and the second captures the total flight time of the algorithm. For simplicity, we will normalize our units of time so that $\tau_0 = 1$. 

\textbf{Trajectory for $\naivesnake$.} Whereas we implement $\naivesnake$ to
travel at a constant speed at each point for each round, our analysis will
consider a variant where $\naivesnake$ halves its speed each round - that is,
takes $2^i$ measurements at each point for each round; this doubling yields
slightly better bounds on sample complexity, and makes $\naivesnake$ compare
even more favorably compared to $\snakeucb$ in theory.\footnote{In practice, we keep the speed constant between trials because, for uniform sampling, this is more efficient; that is, in both theory and practical evaluations, we choose the variant of $\naivesnake$ perforsm the best} 
This results in a total of $2^i
\tau_0 |\calS|$ measurements per round. For $\naivesnake$, the total sample complexity and total run time are equal, and given by
\begin{align*}
\Truntime = \Tsample = |\calS|\sum_{i = 0}^{\ifin} \tau_i .
\end{align*}
\textbf{Termination Criterion for $\naivesnake$:} 
For an arbitrary number of $k$ emitters, $\naivesnake$ terminates at the first round
$i$ in which the $k$-th largest lower confidence bound of all points $x \in \calS$ is higher than the $(k+1)$-th largest upper confidence bound of all points $x \in \calS$.

\subsection{Main Results for $k = 1$ Emitters\label{sec:app_kone}}

We are now ready to state our main theorems for $k=1$ emitters. Recall the divergence terms 
\begin{align}
\diverg(\mu_1,\mu_2) := \frac{(\mu_2 - \mu_1)^2}{\mu_2},~(\mu_1 < \mu_2)
\end{align}
and, in particular, 
\begin{align*}
\divergx := \frac{(\mu^* -  \mu(x))^2}{\mu^*}, ~(\mu(x) < \mu^*)
\end{align*}
from Section~\ref{sec:analysis}. When term $\divergx$ is small, it is difficult to
distinguish between $x^*$ and $x$. 
The following lemma shows that $\divergx$ approximates the the $\KL$-divergence between the distribution $\Poi(\mu(x))$ and $\Poi(\mu^*)$:
\begin{lemma}\label{lem:KL_lem} There exists universal constants $c_1$ and $c_2$
  such that, for any $\mu_2 \ge \mu_1 > 0$,
  \begin{eqnarray}
    c_1\cdot\mathrm{d}(\mu_1,\mu_2) \le \KL(\Poi(\mu_1),\Poi(\mu_2)) \le c_2\cdot\mathrm{d}(\mu_1,\mu_2),
  \end{eqnarray}
  where $\mathrm{d}(\mu_1,\mu_2) =  (\mu_2 - \mu_1)^{2}/\mu_2$.
\end{lemma}
Up to log factors, the
sample complexities for $\snakeucb$ and $\naivesnake$ in the $k = 1$ case are given by $\Hadapt$ and
$\Hunif$, respectively, below:
\begin{eqnarray}
 \Hadapt := \sum_{x \in  \calS \setminus \{x^*\}}\left( \tau_0 +
    \frac{1}{\divergx}\right) & \text{ and } &
    \Hunif  = |\calS| \left( \tau_0 + \max_{x \in S}
  \frac{1}{\divergx} \right)  \label{eqn:complexity}~,
\end{eqnarray}

Similarly to the definitions in Theorem~\ref{thm:snakeucb_runtime_topk_inmain},
$\Hadapt$ and $\Hunif$ differ in that $\Hadapt$ considers the sum over all these
point-wise complexities, whereas $\Hunif$ replaces this sum with the number of
points multiplied by the worst per-point complexity. $\Hadapt$ can be thought of as
the complexity of sampling each point $x \ne x^*$ the exact number of times to
distinguish it from $x^*$, whereas $\Hunif$ is the complexity of sampling
each point the exact number of times to distinguish the best point from
\emph{every other} point. Note that we always have that $\Hadapt \le \Hunif$, and in fact
$\Hunif $ can be as large as $\Hadapt \cdot |\calS|$.

Our first theorem bounds the \emph{sample complexity} of $\snakeucb$ for the $k
= 1$ case presented in Sec.~\ref{sec:analysis} in the main text.
We recall that the sample complexity is the total time spent at all $x \in \calS_i$ until termination:
\begin{theorem}\label{thm:snakeucb} For any $\delta \in (0,1)$, the following holds with probability at least $1-\delta$:  $\snakeucb$ correctly returns $x^*$, the total sample complexity is bounded by bounded above by
  \begin{eqnarray*}
   \Tsample \lesssim \cardS +   \sum_{x \ne x^*} \frac{\log_+\left(\cardS \log_+\left(\tfrac{1}{\divergx} \right)/\delta\right)}{\divergx} ~=~ \widetilde{O}\left(\cardS + \sum_{x \ne x^*} \frac{\log(\cardS/\delta)}{\divergx}\right),
  \end{eqnarray*}
  and the runtime is bounded above by 
  \begin{align*}
  \Truntime &\lesssim \Tsample +  |\calS|\log_+\left( \log_+\tfrac{\cardS}{\delta} \cdot \max_{x \ne x^*} \tfrac{1}{\divergx}\right)~=~ \widetilde{\calO}\left(\Hadapt\log(\cardS/\delta) + \tau_0 |\calS| \log \left(\tfrac{\Hunif}{|\calS|}\right) \right)~,
  \end{align*}
  where $\widetilde{O}(\cdot)$ hides the doubly logarithmic factors in $1/\divergx$.
\end{theorem}
% Note that in the The logarithmic terms correct for the fact that we want (a)
% an error probability of $\delta$, (b) that we have $|\calS|$ possible best points,
% and (c) that our algorithms have random stopping times, and hence need to be
% correct at all stages over time.
%Due to the doubling measurements and snake trajectory, the following proposition
%shows that total travel time is not significantly larger:
\iffalse
\begin{proposition}\label{thm:snakeucb_runtime}
  For any $\delta \in (0,1)$, the following holds with probability at least $1-\delta$:  $\snakeucb$ correctly returns $x^*$, and the total runtime is bounded by bounded above by
  \begin{align*}
    \Truntime &\lesssim  \left\{\sum_{x \ne x^*} 1 + \frac{\log_+\left(\cardS \log_+\left(\tfrac{1}{\divergx} \right)/\delta\right)}{\divergx}\right\} +    |\calS|\log_+\left( \max_{x \ne x^*} \tfrac{1}{\divergx}\right)\\
    &= \widetilde{\calO}\left(\sum_{x \ne x^*}\left\{\frac{\log(\cardS/\delta)}{\divergx}\right\} + |\calS|\log_+\left( \max_{x \ne x^*} \tfrac{1}{\divergx}\right) \right) \\
    &= \widetilde{\calO}\left(\Hadapt\log(\cardS/\delta) + \tau_0 |\calS| \log(\Hunif)\right)~.
  \end{align*}
\end{proposition}
\fi
Theorem~\ref{thm:snakeucb} is a direct consequence of our more general bound for $k \ge 1$ emitters, given by Theorem~\ref{thm:snakeucb_runtime_topk}, which is proved in Sec.~\ref{sec:snakeucbproof}. The next proposition, proved in
Sec.~\ref{sec:naive_proof}, controls the sample complexity of $\naivesnake$:
\begin{theorem}\label{thm:naivesnake} For any $\delta \in (0,1)$, the following holds with probability at least $1-\delta$:  $\naivesnake$ correctly returns $x^*$, and the total runtime is bounded by bounded above by
  \begin{eqnarray*}
   \Truntime \lesssim |\calS| \cdot \left(\max_{x \ne x^*}  \frac{ \log_+\left(\cardS \log_+\left(\tfrac{1}{\divergx}\right)/\delta\right)}{\divergx}\right) &=& \widetilde{O}\left(\Hunif\log(\cardS/\delta) \right)~.
  \end{eqnarray*}

\end{theorem}

Lastly, we show that our adaptive and uniform sample complexities are near
optimal. We prove the following proposition lower bounding the number of samples any adaptive algorithm must take, in
Sec.~\ref{sec:lower_bound_proof}:
\begin{proposition}\label{prop:lower_bound} There exists a universal constant
  $c$ such that, for any $\delta \in (0,1/4)$, any adaptive sampling which
  correctly identifies the top emitting point $x^*$ with probability at
  least $1-\delta$ must collect at least
  \begin{eqnarray*}
    c \log (1/\delta) \cdot \Hadapt
  \end{eqnarray*}
  samples in expectation. Moreover, any uniform sampling allocation which
  identifies the top emitting point $x^*$ with probability at least
  $1-\delta$ must take at least
  \begin{eqnarray*}
    c \log (1/\delta) \cdot \Hunif
  \end{eqnarray*}
  samples in expectation.
\end{proposition}

%% file: top_k_results.tex
%!TEX root = main_arxiv.tex
\subsection{Analysis for Top-$k$ Poisson Emitters \label{sec:app_genk}}
In this section, we continue our analysis of $\snakeucb$, addressing the full
problem of identifying the $k$ Poisson emitters with the highest emission rates. Our goal
is to identify the unique set
\begin{eqnarray}
\calS^*(k) := \{x \in \calS: \mu(x) \ge\mu^{(k)}\}
  \end{eqnarray}
 To ensure the top-$k$ emitters are unique, we assume that
${\mu^{(k)} > \mu^{(k+1)}}$ (recall that $\mu^{(k)}$ denotes the $k$-th largest
value of $\mux$ among all $x \in \calS$). The complexity of identifying the top-$k$ emitters can then described in terms of the gaps of the divergence terms
\begin{align*}
	\divergtop &= \frac{(\mu(x) - \mu^{(k+1)})^2}{\mu(x)},~(\mu(x) > \mu^{(k+1)})  \quad \text{and } \quad\\%\\\text{ for}~ \mu(x) > \mu^{(k+1)} \quad \text{and}\\
  \divergbot &= \frac{(\mu^{(k)} - \mu(x))^2}{\mu^{(k)}},\quad (\mu(x) < \mu^{(k)}).%,\text{ for}~ \mu(x) < \mu^{(k)}
\end{align*}
For $x \in \calS^*(k)$, $\divergtop$ describes how close the emission rate $\mu(x)$ is to the ``best'' alternative in $\calS \setminus \calS^*(k)$. For $x \in \calS \setminus \calS^*(k)$, $\divergbot$ describes how close $\mu(x)$ is to the mean $\mu^{(k)}$ of the emitter in $\calS^*(k)$ from which it is hardest to distinguish. The analogues of $\Hadapt$ and $\Hunif$ are then
\begin{eqnarray}
\Hadaptk &:=& \sum_{x \in \calS^*(k)} \frac{1}{\divergtop} + \sum_{x
              \in \calS \setminus \calS^*(k)} \frac{1}{\divergbot} \quad \text{and} \\
\Hunifk &:=& |\calS| \cdot \max\left\{\max_{x \in \calS^*(k)} \frac{1}{\divergtop},    \max_{x \in \calS \setminus \calS^*(k)} \frac{1}{\divergbot}\right\} = |\calS| \diverg(\mu^{(k+1)}),\mu^{(k)})
\end{eqnarray}
where the equality follows by noting that the function $(x,a) \mapsto \frac{x}{(x - a)^2}$ is decreasing in $x$ and increasing in $a$ for $x > a$. The following theorem, proved in Sec.~\ref{sec:snakeucbproof},  provides an upper bound on the sample complexity for top-$k$ identification:
	\begin{theorem}\label{thm:snakeucb_runtime_topk}
  For any $\delta \in (0,1)$, the following holds with probability at least $1-\delta$:  $\naivesnake$ correctly returns $\Sstk$, and the total sample complexity is bounded above by
  \begin{align*}
  \Tsample &\lesssim \cardS + \sum_{x \in \calS^*(k)} \frac{\log_+\left(\cardS \log_+\left(\frac{1}{\divergtop}\right)/\delta\right)}{\divergtop}  + \sum_{x \in \calS \setminus \calS^*(k)} \frac{\mu^{(k)} \log_+\left(\cardS \log_+\left(\frac{1}{\divergbot}\right)/\delta\right)}{\divergbot} \\
  &=\widetilde{\calO}\left(\Hadaptk \cdot \log(\cardS/\delta)\right),
  \end{align*}
  and the total runtime is bounded by 
  \begin{align*}
    \Truntime &\lesssim \Tsample + \cardS \log_+ \left(\frac{\log_+\tfrac{\cardS}{\delta}}{\diverg(\mu^{(k+1)},\mu^{(k)})} \right) ~=~ \widetilde{\calO}\left(\Hadaptk \cdot \log \tfrac{\cardS}{\delta} +\cardS \log_+\left( \tfrac{\Hunifk}{\cardS}\right)\right)
    %&=& \calO\left(\Hadapt\log(M/\delta) + \tau_0 M^2 \log(\Hunif)\right)~.
  \end{align*}
	\end{theorem}
	We remark that our sample complexity qualitatively matches standard bounds for
  active top-$k$ identification with sub-Gaussian rewards in the non-embodied setting (see, e.g.~\cite{kalyanakrishnan2012pac}). Our results differ by considering the appropriate modifications for Poisson emissions, as well as accounting for total travel time. Lastly, we have the bound for uniform sampling.
  \begin{theorem}\label{thm:naivesnake_top_k}
  For any $\delta \in (0,1)$, the following holds with probability at least $1-\delta$:  $\naivesnake$ correctly returns $\Sstk$, and the total runtime is bounded by bounded above by
  \begin{eqnarray*}
   \Truntime \lesssim  \widetilde{\calO}\left(\Hunif\right)\cdot \log(\cardS/\delta)~.
  \end{eqnarray*}
  \end{theorem}

%% file: uniform_analysis.tex
%!TEX root = main_arxiv.tex

\subsection{Predictions for Simulations\label{sec:discussion_unif}}
We now return to the case $k=1$. We justify the estimates of the complexity terms $\Hadapt$ and $\Hunif$
provided in Sec.~\ref{sec:analysis} of the main text, where $\mux \sim
\mathrm{Unif}[0,\mubar]$ for $x \ne x^*$, and $|\calS| = 256$. To control the
complexity of $\naivesnake$, we observe that
\begin{align*}
\Hunif = \BigOhTil\left(\max_{x
    \ne x^*} \frac{1}{\divergx}\right) = \BigOhTil\left(
  \frac{\mu^*}{(\mu^* - \max_{x \ne x^*}\mux)^2}\right).
  \end{align*}
It is well known that
that the maximum of $N$ uniform random variables on $[0,1]$ is approximately $1
- \Theta(\frac{1}{N})$ with probability $1 - \Theta(\frac{1}{N})$, which implies
that $\max_{x \ne x^*}\mux \approx (1 - \frac{1}{|\calS|})\mubar \approx \mubar$ with probability at least $1 - \Theta(\frac{1}{|\calS|})$. Hence, the sample complexity of $\naivesnake$ should scale as
\begin{align*}
\BigOhTil\left(\frac{|\calS|\mu_*}{(\mubar - \mu_*)^2}\right)
\end{align*}

On the other hand, the sample complexity of $\snakeucb$ grows as
\begin{align*}
\BigOhTil\left(\sum_{x \ne x^*}\frac{1}{\divergx}\right) = \BigOhTil\left(\sum_{x \ne x^*}\mu^*(\mu^* - \mux)^{-2}\right)
\end{align*}
When $\mux \sim \mathrm{Unif}[0,\mubar]$ are random and $\cardS$ is large, the law of large numbers implies that this term tends to $\BigOhTil\left(\mu^* |\calS| \cdot \mathbb{E}_{ \mux \sim \mathrm{Unif}[0,\mubar]}{(\mu^* - \mux)^{-2}}\right)$. We can then compute
\begin{align*}
\mathbb{E}_{ \mux \sim \mathrm{Unif}[0,\mubar]}{(\mu^* - \mux)^{-2}} &= \frac{1}{\mubar} \int_{0}^{\mubar} \frac{1}{(\mu^* - t)^2}dt~=~ \frac{1}{\mubar} \int_{\mu^* - \mubar}^{\mu^*} \frac{1}{u^2}du\\
&= \frac{1}{\mubar} \frac{1}{\mu^* - \mubar} - \frac{1}{\mu^*}~=~ \frac{1}{\mubar} \frac{\mu^* - (\mu^* - \mubar)}{(\mu^* - \mubar)\mu^*}\\
&= \frac{1}{\mubar} \frac{\mubar}{(\mu^* - \mubar)\mu^*} = \frac{1}{(\mu^* - \mubar)\mu^*},
\end{align*}
Hence, the total complexity scales as
\begin{align*}
\BigOhTil\left(\mu^* |\calS| \cdot \mathbb{E}_{ \mux \sim \mathrm{Unif}[0,\mubar]}{(\mu^* - \mux)^{-2}}\right) = \BigOhTil\left(|\calS| (\mu^* - \mubar)^{-1}\right)
\end{align*}
Therefore, the ratio of the runtimes of $\snakeucb$ to $\naivesnake$ are
\begin{align*}
\frac{\mu^* - \mubar}{\mu^*} = 1 - \frac{\mubar}{\mu^*}.
\end{align*}

%% file: analysis_snake.tex
%!TEX root = main_ieee.tex
\section{Analyzing $\snakeucb$: Proof of Theorem~\ref{thm:snakeucb_runtime_topk}\label{sec:snakeucbproof}}
\subsection{Analysis Roadmap}
%In this section, we prove the $k = 1$ complexity results in
%Theorem~\ref{thm:snakeucb}. Note that, in this case, it is unnecessary to
%maintain the set $\Stop$, since it would be empty until the final round.
To simplify the analysis, we assume  that at round $i$, we take a fresh $\tau_i = 2^i$ samples
(recall we have normalized $\tau_0 = 1$) from each remaining $x \in \calS_i$.\footnote{The analysis is nearly the same
  as if we used the total $2^{i+1} - 1$ samples collected throughout.} 
For $x
\in \calS_i$, $\cnt_i(x)$ denotes the number of counts observed from point $x$
over the interval of length $\tau_i$, and $\muhat_i(x)$ denotes the
empirical average emissions; that is, $\muhat_i(x) = \cnt_i(x)/\tau_i$. With
this notation, our confidence intervals take the following form:
\begin{eqnarray}
\LCB_i(x) := \frac{1}{2^i}U_-\left(\cnt_i(x),\frac{\delta}{4
  |\calS| i^2}\right) ~ \text{and} ~
\UCB_i(x) :=  \frac{1}{2^i} U_+\left(\cnt_i(x),\frac{\delta}{4 |\calS|i^2} \right)
\end{eqnarray}

We first argue that there exists a good event, $\Egood(\delta)$, occuring with
probability at least $1-\delta$, on which the true mean $\mux$ of each pixel $x$
lies between $\LCB_i(x)$ and $\UCB_i(x)$ for all $i$. Moreover, $\LCB_i(x)$
and $\UCB_i(x)$ are contained within the interval defined by $[\UCBbar_i(x),
\LCBbar_i(x)]$, which depends explicitly upon $\mu(x)$, but \textit{not} on $\muhat_i(x)$. To derive $\UCBbar_i(x)$ and $\LCBbar_i(x)$, we begin by deriving high probability upper and lower bounds $\overline{U}_+(\mu,\delta)$ and $\overline{U}_-(\mu,\delta)$  for the functions $U_+(\cnt,\delta)$ and $U_-(\cnt,\delta)$ that hold for Poisson random variables. Formally, we have the following
\begin{proposition}\label{prop:poisson_conc2}
  Let $\mu \ge 0$ and let and $\cnt \sim \Poi(\mu)$. Define
  \begin{eqnarray*}
    \overline{U}_+(\mu,\delta) &:=&  \mu + \frac{14}{3} \log\left(1/\delta\right) + 2\sqrt{2\mu\log\left(1/\delta\right)}  ~~\mathrm{and}\\
    \overline{U}_-\left(\mu,\delta\right) &:=& \max\left\{0,\mu - 2\sqrt{2 \mu \log\left(1/\delta\right)} \right\}~.
  \end{eqnarray*}
  Then, it holds that
  \begin{eqnarray}
    \Pr\left[ \overline{U}_-(\mu,\delta) \le U_-(\cnt,\delta) \le \mu \le U_+(\cnt,\delta) \le \overline{U}_+(\mu,\delta)\right] \ge 1 - 2\delta~.
  \end{eqnarray}
\end{proposition}
As a consequence of Proposition~\ref{prop:poisson_conc2}, we can show that
$\LCBbar_i(x)$ and $\UCBbar_i(x)$ are probabilistic lower and upper bounds on $\LCB_i(x)$ and $\UCB_i(x)$:
\begin{lemma}\label{lem:egood} Introduce the confidence intervals
\begin{eqnarray*}
\LCBbar_i(x) := \frac{1}{\tau_i}\overline{U}_-\left(\tau_i\mu(x),\frac{\delta}{4|\calS| i^2}\right) ~ \text{and} ~
\UCBbar_i(x) :=  \frac{|S|}{\tau_i} \overline{U}_+\left(\tau_i\mu(x),\frac{\delta}{4|\calS| i^2} \right)~.
\end{eqnarray*}
Then, there exists an event $\Egood$ for which $\Pr[\Egood] \ge 1 - \delta$, and
\begin{eqnarray}
\forall i \ge 1, x \in \calS_i : \LCBbar_i(x) \le \LCB_i(x) \le \mu(x) \le \UCB_i(x) \le \UCBbar_i(x)~.
\end{eqnarray}
\end{lemma}

Lemma~\ref{lem:egood} is a simple consequence of Propositions~\ref{prop:poisson_conc1}~and~\ref{prop:poisson_conc2}, and a union bound; it is proved formally in Sec~\ref{sec:egood:proof}.
Note that on $\Egood$, one has that $\mu(x) \in [\LCB_i(x), \UCB_i(x)]$ for all rounds $i$ and all $x \in \calS_i$; hence, by Lemma~\ref{lem:main_correctness}, 
\begin{lemma}\label{lem:analysis_correct} If $\Egood$ holds, then for all rounds $i$, $\Stop_i \subset \Sstk \subset \Stop_i \cup \calS_i$; in particular, if $\snakeucb$ terminates at round $\ifin$, then it correctly returns $\Sstk$.
\end{lemma}
Finally, the next lemma, proven in~\ref{sec:stopsample:proof}, gives a \emph{deterministic} condition under which a point $x \in \calS$ can be removed from $\calS_i$, in terms of the deterministic confidence bounds $\LCBbar_i(x)$ and $\UCBbar_i(x)$. 
\begin{lemma}\label{lem:stop_sampling} Suppose $\Egood$ holds. Let $x^{(k)}$ and $x^{(k+1)}$ denote arbitrary points in $\calS$ with $\mu(x^{(k)}) = \mu^{(k)}$ and $\mu(x^{(k+1)}) = \mu^{(k+1)}$. Define the function 
\begin{align*}
\ifin(x) := \begin{cases} \inf\{i: \LCBbar_i(x) > \UCBbar_i(x^{(k+1)})\} & x \in \Sstk\\
\inf\{i: \UCBbar_i(x) < \LCBbar_i(x^{(k)})\} & x \in \calS \setminus \Sstk
\end{cases}
\end{align*}
Then, on $\Egood$,  $x \notin \calS_i$ for all $i > \ifin(x)$.
\end{lemma}
In view of Lemma~\ref{lem:stop_sampling}, we can bound
\begin{align}
\Tsample &:= \sum_{i = 0}^{\ifin}\tau_i |\calS_i|~= ~\sum_{x \in \calS}\sum_{i \ge 0} \tau_i \I(x \in \calS_i) ~\le~ \sum_{x \in \calS}\sum_{i = 0}^{\ifin(x)} \tau_i \quad(\text{by Lemma}~\ref{lem:stop_sampling}) \nonumber\\
&\le \sum_{x \in \calS} 2^{\ifin(x) + 1}\quad (\text{since }\tau_i = 2^i ) \label{eq:tsample_bound}
\end{align}
and further, bound 
\begin{align}
\Truntime &:= \Tsample +  \sum_{i=0}^{\ifin}\tau_0|\calS \setminus \calS_i| ~\le~ \Tsample +  |\calS| \ifin \quad(\text{recall}~\tau_0 = 1) \nonumber\\
&\le \Tsample +  |\calS| \max_{x \in \calS}\ifin(x) \label{eq:truntime_bound}.
\end{align}
where again the last line uses Lemma~\ref{lem:stop_sampling}. Lastly, we prove an upper bound on $\ifin(x)$ for all $x \in \calS$, which follows from algebraic manipulations detailed in Sec.~\ref{sec:sample_comp_prop}:
\begin{proposition}\label{prop:sample_comp} There exists a universal constant $C > 1$ such that, for $x \in \Sstk$,
\begin{align*}
2^{\ifin(x)} \le  C\cdot\left\{1+\frac{\log_+\left(\cardS \log_+\left(\frac{1}{\divergtop}\right)/\delta\right)}{\divergtop}\right\}
\end{align*}
whereas for $x \in \calS \setminus \Sstk$,
\begin{align*}
2^{\ifin(x)} \le  C\cdot\left\{1+\frac{\log_+\left(\cardS \log_+\left(\frac{1}{\divergtop}\right)/\delta\right)}{\divergtop}\right\}
\end{align*}
\end{proposition}
Theorem~\ref{thm:snakeucb_runtime_topk} now follows by plugging in Propostion~\ref{prop:sample_comp} into Equations~\eqref{eq:tsample_bound} and~\eqref{eq:truntime_bound}.

\subsection{Proof of Proposition~\ref{prop:sample_comp}\label{sec:sample_comp_prop}}
Let $n = 2^i$, let $\delbar = \delta/(4|\calS|\log_2 e)$, and let $\Delta = \mu(x_1) - \mu(x_2)$. Then $\UCBbar_i(x_2) < \LCBbar_i(x_1)$ is equivalent to
\begin{eqnarray*}
&&\mu(x_2) + \frac{14}{3n} \log\left(\delbar^{-1}\log n\right) + 2\sqrt{2\mu(x_2)\log\left(\delbar^{-1}\log n\right)/n} < \mu(x_1) - 2\sqrt{2 \mu(x_1) \log\left(\delbar^{-1}\log n\right)/n} \\
&&\impby ~~\frac{14}{3n}\log\left(\delbar^{-1}\log n\right) + 4\sqrt{2\mu(x_1)\log\left(\delbar^{-1}\log n\right)/n} < \Delta~,
\end{eqnarray*}
where the second line uses $\mu(x_2) \le \mu(x_1)$. For the second line to hold, it is enough that
\begin{eqnarray}\label{eq:suff_analysis}
\frac{1}{n}\log\left(\delbar^{-1}\log n\right) < \frac{3\Delta}{28} \text{ and } \log\left(\delbar^{-1}\log n\right)/n < \left(\frac{\Delta}{8 \sqrt{2}}\right)^2
\end{eqnarray}
We now invoke an inversion lemma from the best arm identification literature (see, e.g. Equation (110) in \cite{simchowitz2016best}).
\begin{lemma}\label{lem:inv} For any $\delta,u > 0$, let $\calT(u,\delta) := 1 +  \log_+(\delta^{-1}\log_+(u))/u$. There exists a universal constant $C_0$ such that, for all $n \ge C_0 \calT(u,\delta)$, we have $\log (\delta^{-1} \log n)/n < u$.
\end{lemma}
Hence, Lemma~\ref{lem:inv} and \eqref{eq:suff_analysis} imply that it is sufficent that
\begin{eqnarray*}
n \ge C_0\calT\left(\min\left\{\frac{3\Delta}{28}, \frac{1}{\mu(x_1)}\left(\frac{\Delta}{8 \sqrt{2}}\right)^2\right\},\delta_0\right)~,
\end{eqnarray*}
from which it follows that
\begin{eqnarray*}
\inf \left\{2^i: \UCBbar_{i}(x_2) < \LCBbar_{i}(x_1)\right\} \le 2C_0\left\{ \calT\left(\min\left\{\frac{3\Delta}{28}, \frac{1}{\mu(x_1)}\left(\frac{\Delta}{8 \sqrt{2}}\right)^2\right\},\delbar\right)~\right\}
\end{eqnarray*}
Absorbing constants and plugging in $\delbar = \delta/(4|\calS|\log_2 e)$, algebraic manipulation finally implies that there exists a universal constant $C$ such that, for any $x_1,x_2$ with $\mu(x_1) < \mu(x_2)$,
\begin{eqnarray*}
\inf \left\{2^i: \UCBbar_{i}(x_1) < \LCBbar_{i}(x_2)\right\} &\le& C\left\{1+ \calT\left(\min\left\{\Delta,\frac{\Delta^2}{\mu(x_1)}\right\},\delta/M\right)\right\} \\
&=& C\left\{1 + \calT\left(\frac{\Delta^2}{\mu(x_1)},\delta/\cardS\right)\right\} ~\\
&=& C\left\{1 + \calT\left(\diverg(\mu(x_2),\mu(x_1)),\delta/\cardS\right)\right\} ~.
\end{eqnarray*}
To conclude, we select $x_1 = x^{(k+1)}$ and $x_2 = x$ when $x \in \Sstk$, and $x_1 = x$ and $x_2 = x^{(k)}$ for $x \in \calS \setminus \Sstk$.
\subsection{Proof of Lemma~\ref{lem:egood}
\label{sec:egood:proof}}
\begin{eqnarray*}
&&\Pr\left[\exists x \in \calS_0, i \ge 1: \left\{\LCBbar_i(x) \le \LCB_i(x) \le \mux \le \UCB_i(x) \le \UCBbar_i(x)\right\} \text{ fails}\right]\\
&\overset{\text{union bound}}{\le}&  \sum_{x \in \calS_0, i \ge 1} \Pr\left[\left\{\LCBbar_i(x) \le \LCB_i(x) \le \mux \le \UCB_i(x) \le \UCBbar_i(x)\right\} \text{ fails}\right]\\
&\overset{\text{Prop.}~\ref{prop:poisson_conc1}\&\ref{prop:poisson_conc2}}{\le}&  \sum_{x \in \calS_0, i \ge 1} \frac{\delta}{2|\calS|i^2} 
~=~  |\calS|\sum_{i \ge 1} \frac{\delta}{2|\calS|i^2} 
~=~  \frac{\delta}{2} \sum_{i}i^{-2} ~\le~ \delta~
.
\end{eqnarray*}

\subsection{Proof of Lemma~\ref{lem:stop_sampling}\label{sec:stopsample:proof}}
Assume $\Egood$ holds, and let $x \in \calS$, and set $i = \ifin(x)$. Then
\begin{itemize}
  \item[(a)] If $x \in \Sstk$, then $\LCBbar_i(x) > \UCBbar_i(x^{(k+1)})$. In this case, we shall show that $x$ will be added to $\Stop_i$ via~\eqref{eq:top_elim}. 
  \item[(b)] If $x \in \calS \setminus \Sstk$, then $\UCBbar_i(x) < \LCBbar_i(x^{(k)})$. In this case, we shall show that $x$ will be removed from $\calS_i$ via~\eqref{eq:bot_elim}. 
\end{itemize}

\textbf{Case 1: $\LCBbar_i(x) > \UCBbar_i(x^{(k+1)})$}  By~\eqref{eq:top_elim}, $x$ is added to $\Stop_{i+1}$ if $\LCB_i(x)$ is larger than all but $k - |\Stop_i|$ values of $\UCB_i(x')$, $x' \in \calS_i$. Since $\LCB_i(x) \ge \LCBbar_i(x)$ and $\UCB_i(x') \le \UCBbar_i(x')$ on $\Egood$, it is enough that 
\begin{align*}
\LCBbar_i(x) > (k - |\Stop_i| + 1)\text{-st largest value of }\UCBbar_i(x'),~x' \in \calS_i.
\end{align*}
We now observe that $\UCBbar_i(x')$ is monotonic in $\mu(x')$. Hence, it is enough that
\newcommand{\xiplus}{x_+}
\newcommand{\ximinus}{x_-}
\begin{align*}
\LCBbar_i(x) > \UCBbar_i(\xiplus), \text{ where } \mu(\xiplus) = (k - |\Stop_i| + 1)\text{-st largest value of }\mu(x'),~x' \in \calS_i.
\end{align*}
But since there are exactly $k - |\Stop|$ elements of $\Sstk$ in $\calS_i$ by Lemma~\ref{lem:analysis_correct}, $\xiplus$ is not among the top $k$, and thus $\mu(\xiplus) \le \mu(x^{(k+1)})$. Hence,  $\UCBbar_i(\xiplus) \le \UCBbar_i(\mu(x^{(k+1)}))$, so it is enough that $\LCBbar_i(x) > \UCBbar_i(\mu(x^{(k+1)}))$. 

\textbf{Case 2: $\UCBbar_i(x) < \LCBbar_i(x^{(k)})$} Following the reasoning of case 1 applied to~\eqref{eq:bot_elim}, we can see that it is enough that 
\begin{align*}
\UCBbar_i(x) < \LCBbar_i(\ximinus), \text{ where } \mu(\ximinus) = (k - |\Stop_{i+1}| )\text{-st largest value of }\mu(x),~x' \in \calS_i \setminus \Stop_{i+1}.
\end{align*}
Lemma~\ref{lem:analysis_correct} ensures that there are $(k - |\Stop_{i+1}|)$ members of $\Sstk$ in $\calS_i \setminus \Stop_{i+1}$, so $\mu(\ximinus) \ge \mu(x^{(k)})$. Hence, it is enough that $\UCBbar_i(x) < \LCBbar_i(x^{(k)})$.

%% file: analysis_naive.tex
%!TEX root = main_arxiv.tex
\section{Analysis of $\naivesnake$: Proof of Theorem~\ref{thm:naivesnake}\label{sec:naive_proof}}
In this section, we present a brief proof of Theorems~\ref{thm:naivesnake} and~\ref{thm:naivesnake_top_k}. The arguments are quite similar to those in the analysis of $\snakeucb$, and we shall point out modifications as we go allow.

Let $\Egood$ denote the event of Lemma~\ref{lem:egood}, modified to hold for all $x \in \calS$ at each round $i$ (rather than all $x \in \calS_i$, as in the case of $\snakeucb$). The proof of Lemma~\ref{lem:egood} extends to this case as well, yielding that 
\begin{align*}
\Pr[\Egood] \ge 1 - \delta
\end{align*}
It suffices to show that on $\Egood$, $\naivesnake$ correctly returns $\Sstk$, and satisfies the desired runtime guarantees.

\textbf{Correctness:} On $\Egood$, we have that $x \in \Sstk$, $\mu(x) \le \UCB_i(x)$, and for $x' \in \calS - \Sstk$, $\mu(x') \ge \LCB_i(x')$.  Hence, for any $x' \in \calS - \Sstk$ and for all $x \in \Sstk$,  $\LCB_i(x') \le \mu(x') < \mu(x) \le \UCB_i(x')$. Thus, the termination criterion can only fulfilled when $\LCB_i(x)$, each $x \in \Sstk$, are greater than the remained $\cardS - k$ values of $\UCB_i(x)$. This yields correctness.

\textbf{Runtime:} Recall that for $\naivesnake$ with the standardization $\tau_0$, we have
\begin{align*}
\Truntime = |\calS|\sum_{i = 0}^{\ifin} \tau_i = |\calS|\sum_{i = 0}^{\ifin}2^i \le 2|\calS|\cdot 2^{\ifin}
\end{align*}
Arguing as in the analysis for $\snakeucb$, it suffices to show that, on $\Egood$,
\begin{align}\label{eq:suff_naive}
\ifin \le \inf\{i: \LCBbar_i(x) > \UCBbar_i(x') \quad \forall x \in \Sstk, x' \in \calS - \Sstk\} := \overline{\ifin},
\end{align}
for we bound the the bound 
\begin{align*}
\Truntime &\lesssim \cardS 2^{\overline{\ifin}} \\
&\overset{(i)}{\lesssim} \cardS \cdot \max_{x \in \Sstk, x' \in \calS - \Sstk} \frac{\log_+\left(\tfrac{\cardS}{\delta} \log_+(\frac{1}{\diverg(\mu(x'),\mu(x))})\right)}{\diverg(\mu(x'),\mu(x))}\\
&=  \cardS \cdot \frac{\log_+\left(\tfrac{\cardS}{\delta} \log_+\left(\frac{1}{\diverg(\mu^{(k+1)},\mu^{(k)})}\right)\right)}{\diverg(\mu^{(k+1)},\mu^{(k)} }\\
&= \widetilde{O}(\Hunif\log(\cardS/\delta)),
\end{align*}
where (i) follows from the same argument as in the proof of Proposition~\ref{prop:sample_comp}. To verify~\eqref{eq:suff_naive}, suppose that $\Egood$ holds, and that $\naivesnake$ has not terminated before round $\overline{\ifin}$. Then, by definition of $\overline{\ifin}$,
\begin{align*}
\LCBbar_i(x) > \UCBbar_i(x'),~ \forall x \in \Sstk, x' \in \calS - \Sstk
\end{align*}
Moreover, on $\Egood$, $\LCB_i(x) \ge \LCBbar_i(x)$ and $\UCB_i(x') \le \UCBbar_i(x')$ for all $x \in \Sstk$ and all $x' \in \calS - \Sstk$. Thus,
\begin{align*}
\LCB_i(x) > \UCB_i(x'),~ \forall x \in \Sstk, x' \in \calS - \Sstk,
\end{align*}
which directly implies the termination criterion for $\naivesnake$.

%% file: concentration_proofs.tex
%!TEX root = main_arxiv.tex

\section{Concentration Proofs\label{sec:Concentration_Proofs}}

It is well known that the upper Poisson tail satisfies Bennet's inequality, and its lower tail is sub-Gaussian, yielding the following exponential tail bounds (see, e.g. \cite{boucheron2013concentration}):
\begin{lemma}\label{lem:basic_pois_con} Let $\cnt \sim \Poi(\mu)$. Then,
		\begin{eqnarray}
		\Pr[ \cnt \ge \mu + x] \le \exp\left( - \frac{x^2}{2(\mu + x/3)}\right) \text{ and }  \Pr[ \cnt \le \mu - x]  \le \exp\left( - \frac{x^2}{2\mu}\right)
		\end{eqnarray}
		\end{lemma}
\subsection{Proof of Proposition~\ref{prop:poisson_conc1}\label{sec:poisson_conc1}}
	
\textbf{Proof that $\Pr[\mu \le U_+(\cnt,\delta)] \ge 1 - \delta$:} Recall the definition
\begin{align*}
 U_+\left(\cnt,\delta\right):= 2 \log(1/\delta) + \cnt + \sqrt{2\cnt \log\left(1/\delta\right)}.
\end{align*}
We begin by bounding the lower tail of $\cnt$, which corresponds to the upper confidence $U_+$ bound on $\mu$. Let $\calE_+(\delta)$ denote the event $\{\cnt \ge \mu - \sqrt{2\mu \log(1/\delta)}\}$. By Lemma~\ref{lem:basic_pois_con}, we have that $\Pr[\calE_+(\delta)^c] \le \delta$; hence, it suffices to show that 
\begin{align*}
\calE_+(\delta) \text{ implies } \{\mu \le U_+(\cnt,\delta)\}.
\end{align*}
This follows since by the definition $\calE_{+}(\delta)$ holds, the quadratic equation implies
		\begin{eqnarray}
		\mu^{1/2} &\le& \frac{\sqrt{2 \log(1/\delta)} + \sqrt{2 \log(1/\delta) + 4\cnt }}{2}
		\end{eqnarray}
		Hence, we have 
		\begin{eqnarray*}
		\mu &\le& \frac{2 \log(1/\delta) + 2 \log(1/\delta) + 4\cnt + 2\sqrt{2 \log(1/\delta)}\sqrt{4\cnt + 2\log(1/\delta)} }{4}\\
		&\le& \frac{4 \log(1/\delta) + 4\cnt + 4 \log(1/\delta) + 4\sqrt{2\cnt\log(1/\delta)}}{4}\\
		&=& 2\log(1/\delta) + \cnt + \sqrt{2\cnt\log(1/\delta)} =  U_+(\cnt,\delta)~, \text{as needed.}
		\end{eqnarray*}
		\linebreak
		\textbf{Proof that $\Pr[\mu \ge U_-(\cnt,\delta)] \ge  1 - \delta$:} Recall the definition
		\begin{align*}
		U_-\left(\cnt,\delta\right) := \max\left\{0,\cnt  - \sqrt{2\cnt\log(1/\delta)}\right\}
		\end{align*}
		Analogous to the above, let $\calE_-(\delta):= \{\cnt \le \mu + \sqrt{2 \mu \log(1/\delta)} + \frac{2}{3}\log(1/\delta)\}$. Since $\Pr[ \cnt \ge \mu + x] \le \exp( - \frac{x^2}{2(\mu + x/3)})$, we have that with probability at least $\Pr[\calE_-(\delta)] \ge 1- \delta$. Thus, again it suffices to show that
		\begin{align*}
		\calE_-(\delta) \text{ implies } \{\mu \ge U_-(\cnt,\delta)\}.
		\end{align*}
		We have two cases:
		\begin{itemize}
		\item[(a)] $\cnt \le \frac{2}{3} \log(1/\delta)$. Then, $U_{-}(\cnt,\delta) = 0 $, so $\mu \ge U_-(\cnt,\delta)$ trivially.
		\item[(b)] Otherwise, by solving the quadratic in the definition $\calE_-(\delta)$,
		we find that on $\calE_-(\delta)$,
		\begin{eqnarray*}
		\mu^{1/2} &\ge&  \frac{ -\sqrt{2 \log(1/\delta)} \pm \sqrt{  2\log(1/\delta) - \frac{8}{3} \log (1/\delta) + 4\cnt }}{2}\\
		&=& \frac{ -\sqrt{2 \log(1/\delta)} \pm \sqrt{  4\cnt - 2/3 \log(1/\delta)}}{2}~,
		\end{eqnarray*}
		where we note that the discriminant is positive since $\cnt \ge \frac{2}{3}\log(1/\delta)$. Squaring, we have
		\begin{eqnarray*}
		\mu &\ge& \frac{ 2 \log(1/\delta) + 4\cnt - 2\log(1/\delta)/3  - 2\sqrt{2\log\delta^{-1}}\sqrt{4\cnt - 2/3\log(1/\delta)} }{4}\\
		&\ge&  \frac{ 4\cnt + (2 - 2/3)\log(1/\delta) - 4\log(1/\delta)\sqrt{1/3} -  4\sqrt{2\cnt \log \delta^{-1}}}{4}\\
		&\ge& \cnt + (1 - 1/6 - \sqrt{1/3})\log(1/\delta) - \sqrt{2 \cnt \log(1/\delta)} \ge \cnt - \sqrt{2 \cnt \log(1/\delta)} ~.
		\end{eqnarray*}
		Since we also have $\mu \ge 0 $, we see that on $\mathcal{E}_{0}(\delta)$, we have that
		\begin{align*}
		\mu \ge \max\{\cnt - \sqrt{2 \cnt \log(1/\delta)}, 0\} = U_{-}(\cnt,\delta), \text{ as needed}.
		\end{align*}

\end{itemize}
\subsection{Proof of Proposition~\ref{prop:poisson_conc2}}
	From section~\ref{sec:poisson_conc1}, recall the events 
	\begin{eqnarray}\calE_+(\delta) := \{\cnt \ge \mu - \sqrt{2\mu \log(1/\delta)}\} \quad \text{ and } \quad \calE_-(\delta):= \{\cnt \le \mu + \sqrt{2 \mu \log(1/\delta)} + \frac{2}{3}\log(1/\delta)\}~.
	\end{eqnarray} 
	Further, recall that on $\calE_+(\delta)$, we have $\mu \le U_+(\cnt,\delta)$ and $\mu \ge U_-(\cnt,\delta)$. We now show that on $\calE_-(\delta)$, we also have $U_+(\cnt,\delta) \le \UCBbar(\mu,\delta)$, and on $\calE_+(\delta)$, we have $U_-(\cnt,\delta) \ge \LCBbar(\mu,\delta)$. 
	\textbf{Bounding $U_-(\cnt,\delta) \ge \overline{U}_-(\mu,\delta)$:} To bound $U_-(\cnt,\delta) \ge \overline{U}_-(\mu,\delta)$, observe that $\cnt \mapsto U_-(\cnt,\delta)$ is increasing in $\cnt$, and on $\calE_+(\delta)$, one has $\{\cnt \ge \mu - \sqrt{2\mu \log(1/\delta)}\} $. Thus
	\begin{eqnarray*}
	U_-(\cnt,\delta) &:=& \cnt  - \sqrt{2\cnt\log(1/\delta)} \\
	&\ge&  \mu - \sqrt{2\mu \log(1/\delta)} - \sqrt{2  \log(1/\delta)\mu( \mu - \sqrt{2 \mu \log(1/\delta)} } \\ 
	&\ge&  \mu - \sqrt{2\mu \log(1/\delta)} - \sqrt{2  \mu \log(1/\delta)} \\
	&\ge&  \mu - 2\sqrt{2\mu \log(1/\delta)} ~:=~\overline{U}_-(\mu,\delta)~.
	\end{eqnarray*}
	\textbf{Bounding $U_+(\cnt,\delta) \le \overline{U}_+(\mu,\delta)$:} Next, we prove the bound bound $U_+(\cnt,\delta) \le \overline{U}_+(\mu,\delta)$. On $\calE_-(\delta)$, we have
	\begin{eqnarray*}
	\cnt &\le& \mu + \frac{2}{3}\log(1/\delta) +  \sqrt{2\mu\log(1/\delta)} \\
	&=& (\mu^{1/2} + \sqrt{2\log(1/\delta)})^2 - 2\log(1/\delta) + \frac{2}{3}\log(1/\delta) \\
	&\le& (\mu^{1/2} + \sqrt{2\log(1/\delta)})^2 
	\end{eqnarray*}
	Hence, when the above occurs, we have
	\begin{eqnarray*}
	U_+(\cnt,\delta) &=& 2 \log(1/\delta) + \cnt + \sqrt{2\cnt \log(1/\delta)} \\
	&=& 2 \log(1/\delta) + \mu + \frac{2}{3}\log(1/\delta) +  \sqrt{2\mu\log(1/\delta)} + \sqrt{2((\mu^{1/2} + \sqrt{2\log(1/\delta)})^2 ) \log(1/\delta)} \\
	&=& \frac{8}{3} \log(1/\delta) + \mu + \sqrt{2\mu\log(1/\delta)} + \sqrt{2((\mu^{1/2} + \sqrt{2\log(1/\delta)})^2 ) \log(1/\delta)} \\
	&\le& \frac{8}{3} \log(1/\delta) + \mu + \sqrt{2\mu\log(1/\delta)} + \sqrt{2\mu\log(1/\delta)} +2 \log(1/\delta)\\
	&=& \frac{14}{3} \log(1/\delta) + \mu + 2\sqrt{2\mu\log(1/\delta)} 
	\end{eqnarray*}

		% \begin{proof}
		% \begin{eqnarray}
		% \Pr[ \cnt \ge \mu + x] \le \exp( - \mu\Psi(x/\mu)) 
		% \end{eqnarray}
		% Using the bound $\Psi(x/\mu) \ge \frac{(x/\mu)^2}{2 + 2(x/\mu)} = \frac{x^2}{2(\mu + x/3)}$.  Thus 
		% \end{proof}

%% file: lower_bounds.tex
%!TEX root = main_arxiv.tex

\section{Lower Bounds}

\subsection{Proof of Proposition~\ref{prop:lower_bound}\label{sec:lower_bound_proof}}
The basic proof strategy follows along the lines of the information-theoretic lower bounds in Kaufmann et al.~'16~\cite{kaufmann2016complexity}. Consider a grid $\calS$ of $\cardS$ points, with means $\mu(x), x \in \calS$. We fix a given sampling algorithm, adaptive or otherwise, and let $T(x)$ denote the expected number of measurements from point $x$ given that the means are given by $\mu(x)$. Suppose $x^* := \arg \max_{x \in \calS} \mu(x)$ is unique. We will argue that for a universal constant $c_1$ and any $x \ne x^*$~,
\begin{eqnarray}\label{eq:LBWTS1}
T(x) \ge \frac{c_1 \log(1/\delta)}{\KL(\mu(x),\mu(x^*))}~.
\end{eqnarray}
By the $\KL$ approximation in Lemma~\ref{lem:KL_lem}, this implies that for some universal constant $c_2$,
\begin{eqnarray}\label{eq:LBWTS2}
T(x) \ge  \frac{c_2 \log(1/\delta)}{\divergx}~.
\end{eqnarray}
For adaptive sampling, the expected number of samples is at least $\sum_{x \ne x^*} T(x)$, which by~\eqref{eq:LBWTS2} is at least
\begin{eqnarray*}
c_2 \log(1/\delta) \cdot \sum_{x \ne x^*}\frac{\mu(x^*)}{\Delta_x^2}~.
\end{eqnarray*}
This completes the proof for adaptive sampling. For non-adaptive sampling, $T(x) = T(x')$ for all $x,x' \in \calS$. Hence, the expected number of samples is at least
\begin{eqnarray*}
\sum_{x \in \calS} T(x) = \cardS \max_{x \ne x^*} T(x) \overset{\eqref{eq:LBWTS2}}{\ge}   \cardS \cdot \max_{x \ne x^*} c_2\log(1/\delta) \frac{\mu(x^*)}{\Delta_x^2}~.
\end{eqnarray*}

We now verify Equation~\eqref{eq:LBWTS1}. To do so, consider an alternative grid $\calS$ of $\cardS$ pixels, with means $\mu'(x)$. Suppose moreover that $x^{*'} := \arg \max_{x \in \calS}{ \mu'(x)}$ is unique,
and that $x^{*} \ne x$. The key insight from Kaufmann et al.~'16~\cite{kaufmann2016complexity} is that any algorithm which identifies $x^*$ with probability $1-\delta$ must be able to distinguish between the means $\mu'(x)$ and the with means $\mu(x)$. Kaufmann et al.~'16~\cite{kaufmann2016complexity} shows that this requires that the expected number of samples $T(x)$ satisfy
\begin{eqnarray}\label{eq:Kaufman}
 \sum_{x \in \calS} T(x)\KL(\mu(x),\mu'(x)) \ge c_1 \log(1/\delta)
\end{eqnarray}
Now let's fix a particular $x_0 \ne x^*$ and an $\epsilon > 0$. We can define the means $\mu'(x)$ to be
\begin{eqnarray*}
\mu'(x) := \begin{cases} \mu(x^*) + \epsilon & x = x_0 \\
\mu(x) & \text{otherwise}
\end{cases}
\end{eqnarray*}
Note then that $\arg\max_{x} \mu'(x) = x_0$, and $\mu(x_0) = \mu(x^*) + \epsilon$. Hence, Equation~\eqref{eq:Kaufman} holds for the means $\mu'(\cdot)$. Moreover, $\mu(x) = \mu'(x)$ for all $x \ne x_0$, so that $\KL(\mu(x),\mu'(x)) = 0$ for $x \ne x_0$. Hence,  Equation~\eqref{eq:Kaufman} simplifies to
\begin{eqnarray*}
T(x_0) \KL( \mu'(x_0), \mu'(x^*) + \epsilon)  = T(x_0)\KL(\mu(x_0),\mu'(x_0)) \ge c_1 \log(1/\delta)
\end{eqnarray*}
Since $\KL(\mu'(x_0), \mu'(x^*) + \epsilon)$ is continuous in $\epsilon$ (see Fact~\ref{fact:KL} below), taking $\epsilon \to 0$ yields
\begin{eqnarray*}
T(x_0) \KL( \mu'(x_0), \mu'(x^*)) \ge c_1 \log(1/\delta)~ \text{ as needed}.
\end{eqnarray*}

\subsection{Proof of Lemma~\ref{lem:KL_lem}\label{sec:KL_lem_proof}}
We begin by stating a standard computation of the $\KL$-divergence between two Poisson distributions.
\begin{fact}\label{fact:KL} $\KL(\Poi(\mu_1), \Poi(\mu_2)) = \mu_1 \log(\mu_1/\mu_2) + (\mu_2 - \mu_1)$~. \end{fact}

To prove Lemma~\ref{lem:KL_lem}, recall that we assume that $\mu_2 \ge \mu_1$. We may therefore reparameterize $\mu \leftarrow \mu_2$, and $\mu_1 \leftarrow (1-\alpha) \mu_2$ for $\alpha \in (0,1)$. One then has
\begin{eqnarray*}
\KL(\Poi(\mu_1), \Poi(\mu_2)) = \mu \{\alpha + (1- \alpha)\log(1 - \alpha)\}
\end{eqnarray*}
Since $\mu_2 - \mu_1 = \mu( 1 - (1-\alpha)) = \alpha \mu$, it suffices to show that there exists constants $c_1$ and $c_2$ such that
	\begin{eqnarray}\label{eq:WTS1}
	c_1 \mu \alpha^2 \le \KL(\Poi(\mu_1),\Poi(\mu_2)) \le c_2 \mu \alpha^2~.
	\end{eqnarray}
	To this end, it suffices to show that there exists a universal constant $\alpha_0 > 0$, such for all $\alpha \le \alpha_0$, one has
	\begin{eqnarray}\label{eq:WTS2}
	\KL(\Poi(\mu_1), \Poi(\mu_2)) \in \left[\frac{1}{4},\frac{3}{4}\right] \mu \alpha^2~.
	\end{eqnarray}
	Indeed, for any $\alpha \ge \alpha_0$, we have that
	\begin{eqnarray*}
	0 < \mu (\alpha_0 + (1-\alpha_0)\log(1-\alpha_0))  \le  \KL(\Poi(\mu_1), \Poi(\mu_2))  \le \mu~,
	\end{eqnarray*}
	which implies that, for $\alpha \in [\alpha_0,1]$,
	\begin{eqnarray*}
	0 &<& \left(\alpha^2 \mu\right) \cdot (\alpha_0 + (1-\alpha_0)\log(1-\alpha_0))\\
	  &\le& \alpha^2 \mu \cdot \frac{\alpha_0 + (1-\alpha_0)\log(1-\alpha_0)}{\alpha^2}  \\
	  &\le&  \KL(\Poi(\mu_1), \Poi(\mu_2))  \\
	  &\le& \alpha^2 \mu \cdot \frac{1}{\alpha^2}~ \le \left(\mu \alpha^2\right)  \cdot \frac{1}{\alpha_0^2}~.
	\end{eqnarray*}
	Hence taking $c_1 = (\alpha_0 + (1-\alpha_0)\log(1-\alpha_0))$ and $c_2 = \frac{1}{\alpha_0^2}$, we see that~\eqref{eq:WTS1} holds for $\alpha \in [\alpha_0,1]$. We now turn to prove~\eqref{eq:WTS2}. Note that $\log'(1-x) = -1/(1-x)$, $\log''(1-x) = 1/(1-x)^2$, and $\log'''(x) = -2/(1-x)^3$. Hence, by Taylor's theorem, there exists an $\alpha' \in [0,\alpha]$ such that
	\begin{eqnarray*}
	\log(1-\alpha) &=& \alpha \log'(1) + \frac{\alpha^2}{2}\log''(1) + \frac{\alpha^3}{6}\log'''(1-\alpha')\\
	&=& -\alpha  - \frac{\alpha^2}{2} - \frac{2\alpha^3}{(1-\alpha')^3}~.
	\end{eqnarray*}
	Hence,
	\begin{eqnarray*}
	\KL(\Poi(\mu_1), \Poi(\mu_2)) &=& \mu \left\{\alpha - (1- \alpha)(\alpha + \frac{\alpha^2}{2} + \frac{2\alpha^3}{(1-\alpha')^3})\right\} \\
	&=& \mu \left\{\alpha - (1- \alpha)(\alpha + \frac{\alpha^2}{2} + \frac{2\alpha^3}{(1-\alpha')^3})\right\} \\
	&=& \mu \left\{\alpha - \alpha - \frac{\alpha^2}{2} - \frac{2\alpha^3}{(1-\alpha')^3} + \alpha^2 + \frac{\alpha^3}{2} + \frac{2\alpha^4}{(1-\alpha')^3}\right\}\\
	&=& \mu \left\{\frac{\alpha^2}{2} + \alpha^3\left(\frac{-2}{(1-\alpha')^3} + \frac{1}{2} +\frac{2\alpha}{(1-\alpha')^3}\right)\right\}~.
	\end{eqnarray*}
	In particular, there exists a universal constant $\alpha_0$ such that, for all $\alpha \le \alpha_0$,
	\begin{eqnarray*}
	\KL(\Poi(\mu_1), \Poi(\mu_2)) \in \left[\frac{1}{4},\frac{3}{4}\right] \mu \alpha^2~.
	\end{eqnarray*}